\documentclass[reqno]{amsart}

 \usepackage[foot]{amsaddr}
\usepackage{amsthm}
\usepackage{amssymb}
\usepackage{enumerate}
\usepackage{color}
\usepackage{float}
\usepackage{bbm}
\usepackage{mathrsfs}
\usepackage{amsmath}
\usepackage{comment}
\usepackage{listings}
\usepackage[dvipsnames]{xcolor}
\usepackage{fourier}
\usepackage[ruled,noend]{algorithm2e}
\SetKwInOut{Input}{Input}
\SetKwInOut{Output}{Output}
\SetKw{Return}{Return}
\usepackage{hyperref}
\usepackage[capitalise]{cleveref}
\usepackage[colorinlistoftodos]{todonotes}
\usepackage{cite}
\hypersetup{
	colorlinks,
	citecolor=blue,
	filecolor=blue,
	linkcolor=blue,
	urlcolor=blue
}
\allowdisplaybreaks

\newtheorem{theorem}{Theorem}[section]
\newtheorem{proposition}[theorem]{Proposition}
\newtheorem{lemma}[theorem]{Lemma}

\newtheorem{corollary}[theorem]{Corollary}
\newtheorem{definition}[theorem]{Definition}

\theoremstyle{definition}

\newtheorem{question}{Question}

\newtheorem{remark}[theorem]{Remark}

\newcommand{\cF}{\mathcal{F}}

\newcommand{\R}{\mathbb{R}}
\newcommand{\Z}{\mathbb{Z}}

\newcommand{\N}{\mathbb{N}}
\newcommand{\CC}{\mathrm{C}}

\newcommand{\E}{\mathbb{E}}
\newcommand{\Prob}{\mathbb{P}}

\newcommand{\sspan}{\textnormal{span}}

\def\mc{\mathcal}

\newcommand{\one}{\mathbbm{1}}

\newcommand{\dd}{\mathrm{d}}
\def\W{\mathrm{W}}
\newcommand{\LL}{\mathrm{L}}
\newcommand{\BB}{\mathrm{B}}

\DeclareMathOperator*{\dist}{\mathrm{dist}}
\newcommand{\TV}{\mathrm{TV}}
\DeclareMathOperator*{\id}{\mathrm{id}}

\newcommand{\Sob}[2]{\W^{#1,#2}}

\newcommand{\supp}{\operatorname{supp}}

\newcommand{\diam}{\operatorname{diam}}
\newcommand{\rad}{\operatorname{rad}}

\newcommand {\la} {\langle}
\newcommand {\ra} {\rangle}
\newcommand{\argmin}{\operatorname{argmin}}
\newcommand{\bracket}[1]{\left\la#1\right\ra}
\newcommand{\eps}{\varepsilon}
\newcommand{\quanterror}{\mathcal{E}}

\title{Structured approximations of measures}

\begin{document}

\author{Keaton Hamm$^{1,2}$}
\address{$^1$Department of Mathematics, The University of Texas at Arlington, Arlington, TX}
\address{$^2$Division of Data Science, The University of Texas at Arlington, Arlington, TX
}
\email{keaton.hamm@uta.edu}

\author{Varun Khurana$^3$}\address{$^3$Department of Applied Mathematics, Brown University, Providence, RI 02906.}\email{varun\_khurana@brown.edu}


\begin{abstract}
We study the approximation of probability measures in the Wasserstein-$p$ distance by structured classes of approximators, motivated by applications in imaging, machine learning, and physical measurement under sensor constraints. We obtain three sets of results. First, for measures with densities bounded away from zero on a bounded Lipschitz domain $\Omega$, we prove that any approximation scheme for functions in $\LL_p(\Omega)$ transfers, with linear rate, to a corresponding approximation scheme for measures in $\W_p(\Omega)$. The argument applies a theorem of Bogovskii on regularity of solutions to the continuity equation in the Benamou–Brenier formulation of optimal transport.  We exhibit concrete approximation schemes (polynomials, shift-invariant spaces, cardinal interpolation with radial basis functions, kernel density estimators, and piecewise approximations on nonuniform Voronoi partitions) that fit the framework. As a matter of independent interest, we prove a negative Sobolev lower bound that generalizes existing bounds from $p=2$ to all $p\in(1,\infty)$. We also consider deterministic bounds for discrete approximations to arbitrary measures in terms of the mesh norm of a quasi-uniform set of points. We specialize these bounds to show that compactly supported measures admit a deterministic $N$-term approximation $\mu_N$ such that $\W_p(\mu,\mu_N) = O(N^{-\frac1d})$ for all $d\geq 1$, which matches the asymptotic optimal quantizer rate. We also extend these results to non-compactly supported measures with appropriate tail decay.
\end{abstract}

\maketitle

\section{Introduction}

The approximation of functions with various levels of smoothness by structured classes of functions has been studied extensively for a long time, going back at least to the study of approximating functions with polynomials by Chebyshev (see, e.g., \cite{kolmogorov1998mathematics} for a history) and Weierstrass (and later Stone's generalization \cite{stone1948generalized}). Many initial results showed density of a special class of functions (e.g., polynomials of a certain degree) in a larger smoothness space (e.g., $\CC^k(\Omega)$). However, for approximation in practice, one often wants to know the approximation rate of a concrete scheme involving the approximating functions.

In more modern language, let $(\mc{X},d)$ be a metric space and $\mc{D}, \mc{A}\subset \mc{X}$ be \textit{data}, and a structured set of \textit{approximations} to the data, respectively. Then for $x\in\mc{D}$, define its distance from the approximation set in the metric $d$ via
\[\dist(x,\mc{A})_{d} = \inf_{y\in\mc{A}}d(x,y).\]
This quantity captures how well $\mc{A}$ \textit{locally} approximates a data point $x$. Next, define the \textit{approximation power} of $\mc{A}$ in $\mc{D}$ with respect to the metric $d$ via
\[\mc{E}(\mc{D},\mc{A})_d := \sup_{x\in\mc{D}}\dist(x,\mc{A})_d = \sup_{x\in\mc{D}}\inf_{y\in\mc{A}}d(x,y).\]
This is a worst-case error measure that is useful for determining the \textit{global approximation power} of a model class $\mathcal{A}$. 

Typically one bounds $\dist(x,\mc{A})_d$ or $\mc{E}(\mc{D},\mc{A})_d$ in terms of some quantity of relevance to $\mc{D}$, e.g., the smoothness level, and also of $\mc{A}$, e.g., the degree of polynomial. For example, if $\mc{A}$ are Taylor polynomials of degree at most $k$, then one has that for $f\in \CC^{k+1}(\Omega)$ for some domain $\Omega\subset\R^d$, $\dist(f,\mc{A})_{\LL_\infty(\Omega)} \leq \frac{\diam(\Omega)^{k+1}}{(k+1)!}\max_{|\alpha|=k+1}\|D^\alpha f\|_{\LL_\infty(\Omega)}$. Hangelbroek et al.~\cite{hangelbroek2012cardinal} show that for Sobolev functions $f\in\Sob{k}{p}(\R^d)$, cardinal interpolation at the scaled lattice $h\Z^d$ by a shift-invariant space associated with the Gaussian generator $S_p^h(\phi_{h^2})$ (see \Cref{SEC:GaussianCardinal} for definition), $\dist(f,S_p^h(\phi_{h^2}))_{\Sob{k}{p}(\R^d)} \leq C\|f\|_{\Sob{k}{p}(\R^d)}h^k$ where the constant depends on $k,p,$ and $d$. In this instance, the local approximation power is of order $h^k$ where $h$ is the spacing of the interpolation grid.

While function approximation has been well studied and we have a rich, diverse understanding of it, approximation of measures by structured approximation schemes is significantly less understood. Much is known about empirical approximation for measures (see \Cref{SEC:PriorArt} for more details). However, deterministic methods for discrete approximation (in particular schemes in which points cannot be sampled i.i.d. from $\mu$) are comparatively underdeveloped, despite their relevance for applications where the sampling locations are fixed by external constraints.

The main purpose of this paper is threefold: first, to establish a quite general framework for transferring approximation rates from function approximation schemes to approximation of measures with smooth densities; second, to exhibit rates of approximation for piecewise approximation schemes subordinate to a Voronoi decomposition from a nonuniform mesh; and third, to prove discrete $N$-term approximation rates for arbitrary measures (including discrete ones without densities) that do not require empirical sampling of the target measure.

\subsection{Approximation of measures with smooth densities}

Of interest to us here is the case $(\mc{X},d) = (\W_p(\Omega),\W_p)$, which is the Wasserstein space over some fixed domain $\Omega\subset\R^d$. Our main question is how to transfer known approximation results from function approximation theory to measure approximation. In particular, we ask the following question.

\begin{question}\label{Q:Smoothness}
Let $\Omega\subset\R^d$, $p\in[1,\infty)$, and let $\cF\subset \LL_p(\Omega)$ be a set of smooth functions (e.g., Sobolev functions).  Let $\mathcal{M}_{\mc{F}}\subset \W_p(\Omega)$ consist of absolutely continuous measures whose densities lie in $\mc{F}$, and we use the slight abuse of notation $\mc{M}_{\mc{F}} := \{f\dd x:f\in\mc{F}\}$. Suppose $\mathcal{A}$ is an approximation space approximating $\cF$, and thus $\mc{M}_{\mc{A}} = \{g\dd x:g\in\mathcal{A}\}$ approximates $\mathcal{M}_{\mc{F}}$. Can one transfer approximation power of $\cF$ with $\mathcal{A}$ to $\mc{M}_{\mc{F}}$ with $\mathcal{M}_{\mc{A}}$? More concretely, under what conditions do we have \begin{equation}\label{EQN:ApproximationPower}\mc{E}(\mathcal{M}_{\mc{F}},\mc{M}_{\mathcal{A}})_{\W_p} \leq C \mc{E}(\cF,\mathcal{A})_{\LL_p}\end{equation} where $C$ is some constant that is quantifiable in terms of properties of $\mc{A}$ and $\mc{F}$?
\end{question}

\begin{question}\label{Q:Local}
    The approximation error above is a global worst case approximation measure, but some approximation schemes are known to achieve good approximation rates locally. Therefore, we also ask if for every $f\in\mc{F}$, there is a constant $C_f$ which may depend on $f$ such that
    \begin{equation}\label{EQN:LocalApproximationPower}
        \dist(f\dd x, \mc{M}_{\mc{A}})_{\W_p} \leq C_f\dist(f,\mc{A})_{\LL_p}?
    \end{equation}
\end{question}

If the answer to \Cref{Q:Smoothness} is positive (i.e., \eqref{EQN:ApproximationPower} holds), then we say that \textit{global approximation power transfers} from $\mc{F}$ with $\mc{A}$ to $\mc{M}_{\mc{F}}$ with $\mc{M}_{\mc{A}}$. If \eqref{EQN:LocalApproximationPower} holds, then we say that \textit{local approximation power transfers}.

Our main result resolves Questions \ref{Q:Smoothness} and \ref{Q:Local} affirmatively under minimal assumptions on the target and approximating measures. In the case $p=1$, all that is needed is that $\Omega$ is bounded, whereas for general $p$, $\Omega$ must be a bounded Lipschitz domain (\Cref{DEF:LipschitzDomain}), and the functions $\mc{F}$ and $\mc{A}$ must be bounded away from $0$. Under these mild conditions, we find that global and local approximation power transfers. We state informally this result here.

\begin{theorem}[Informal statement of \Cref{THM:ApproximationTransfer}]\label{THM:ApproximationTransferIntro}
Let $p\in(1,\infty)$, and $\Omega\subset\R^d$ be a bounded Lipschitz domain. Let $\mc{F},\mc{A}\subset \LL_p(\Omega)$ consist of densities bounded uniformly away from 0. Then both \[\mc{E}(\mathcal{M}_{\mc{F}},\mc{M}_{\mathcal{A}})_{\W_p} \leq C \mc{E}(\cF,\mathcal{A})_{\LL_p}\quad \textnormal{and}\quad\dist(f\dd x,\mc{M}_{\mc{A}})_{\W_p} \leq C\dist(f,\mc{A})_{\LL_p}\] for some $C=C_{\Omega,d,p,m}$ (independent of $f$ in the second bound). If $p=1$, both inequalities hold without the assumption that the densities are bounded below, and for bounded (not necessarily convex) $\Omega$.
\end{theorem}

The proof of this result uses a result of Bogovskii \cite{bogovskii1979solution}  to exhibit control of solutions to the continuity equation in the dynamic formulation of optimal transport. We are not aware of prior work that applies Bogovskii's theorem to obtain upper bounds on the Wasserstein distance via the dynamic formulation. This approach has the advantage of working on bounded domains with rough boundaries (the result holds for John domains; see the discussion following \Cref{THM:ApproximationTransfer}). Thus, the result above holds in a quite general setting; in particular, we do not need to assume smoothness of the boundary or convexity of $\Omega$.

In \Cref{SEC:Examples}, we describe several examples of concrete approximation schemes that fit into the framework above: Taylor polynomial approximation, shift-invariant space approximations, cardinal Gaussian interpolation, and kernel density estimation (KDE).

When $\Omega$ is convex and densities $f$ and $g$ are bounded above almost everywhere, we also prove a lower bound in a negative Sobolev norm for general $p$, which generalizes Peyre's $\dot{H}^{-1}$ bound for $\W_2$. If $f,g\leq M$, the bound is of the form
\[M^{-\frac{1}{p'}}\|f-g\|_{\Sob{-1}{p}_\diamond(\Omega)} \leq \W_p(f\dd x,g\dd x)\]
(see \Cref{SEC:LowerBound} for details).

\subsection{Piecewise approximations of densities}

\Cref{THM:ApproximationTransferIntro} is about global approximation of smooth functions on bounded Lipschitz domains. However, in practice, we often wish to form piecewise approximations of a density on some decomposition of $\Omega$.  To this end, we consider a quasi-uniform set of points $X:=\{x_i\}_{i=1}^\infty\subset\R^d$ and its corresponding Voronoi partition $\{V_i\}_{i=1}^\infty\subset\R^d$. For our purposes, we will assume that $\R^d=\bigcup_i V_i$ and that the Voronoi partition can be made disjoint: $V_i\cap V_j=\emptyset$, $i\neq j$ (one can enforce this for arbitrary quasi-uniform $X$; see \Cref{SEC:BackgroundVoronoi}). Our main result in this direction is the following.

\begin{theorem}[Informal statement of \Cref{thm:piecewise-approximation}]
\label{thm:piecewise-approximation-intro}
Let $p\in(1,\infty)$ and $\Omega\subset\R^d$ be a bounded Lipschitz domain. Suppose $f\in\LL_p(\Omega)$ is a density bounded away from 0, and that $f$ is approximated by $A_i(f)$ locally on each cell of a quasi-uniform Voronoi partition with local $\LL_\infty(V_i)$
error sufficiently small. Then the piecewise approximation, after renormalization to a density $g$, satisfies
\[\W_p(f\dd x,g\dd x) \lesssim \|f\|_{\LL_p(\Omega)}\left(\sum_{i=1}^N \|f - A_i(f)\|^p_{\LL_p(V_i)}\right)^{\frac1p}.\]
If $p=1$, $\Omega$ need not be Lipschitz and $\|f\|_{\LL_1(\Omega)}=1$ in the final bound.
\end{theorem}

The piecewise approximation framework is also well-suited to adaptive schemes. 
The global error above is in terms of the local contributions of  $\|f - 
A_i(f)\|_{\LL_p(V_i)}^p$, each of which can be bounded in terms of the local 
$\LL_\infty$ approximation error and the volume of the Voronoi cell $V_i$. This suggests a computational tradeoff in which finer meshes and higher-order approximators can be deployed 
where $f$ is difficult to approximate, and coarser, cheaper approximators where 
$f$ is well behaved, with the overall Wasserstein error controlled throughout (see \Cref{REM:EUpperBound}).

\subsection{Voronoi approximations of arbitrary measures}

\Cref{thm:piecewise-approximation-intro} exhibits how to obtain approximation error rates for piecewise approximations of (typically) smooth functions. However, in practice one often observes only discrete samples of a measure which may not be smooth. There are two main types of discrete approximations considered in the literature for measures in $\W_p(\R^d)$: optimal quantizers and empirical measures. Optimal quantization rates are only asymptotic, whereas empirical measure convergence rates are in expectation or probability, and typically achieve the same asymptotic rate (usually $O(N^{-\frac1d})$, but see \Cref{SEC:PriorArt} for a discussion). Optimal quantization asks where the best locations of samples of the measure are to give optimal $N$-term approximation in $\W_p$. These are informative to give an ideal bound, but such schemes may not always be practical. In contrast, empirical measure approximation requires that one be able to sample points i.i.d. from the underlying measure, but this may not be feasible in some applications. 

Indeed, there may be physical or economical constraints that do not allow one to sample from an underlying measure directly. For example, consider approximating precipitation intensity over a geographic area based on rainfall gauges. Natural constraints on placement of gauges include terrain, human infrastructure, and cost (limiting density of gauges). In such settings, the sensor locations $\{x_i\}$ are determined by external constraints rather than the measure itself, yet one still wishes to approximate the underlying distribution. To that end, we consider discrete approximations to arbitrary measures (with no smoothness or absolute continuity assumption) in $\W_p(\R^d)$ at quasi-uniform points $X$ (see \Cref{SEC:Voronoi}).

The following are our main results from \Cref{SEC:Voronoi}.

\begin{theorem}[Informal statement of \Cref{THM:Nonuniform,COR:NonuniformCompact}]
\label{thm:voronoi-intro}
Let $\mu\in\W_p(\R^d)$, $p\in[1,\infty)$, and let 
$X\subset\R^d$ be a discrete point set with mesh norm $h_X<\infty$. For each $i$, 
let $\tau_i$ be any non-negative measure supported on the Voronoi cell $V_i$ of $X$ with $\tau_i(V_i)=0$ only when $\mu(V_i)=0$.
Then the structured approximation $\mu_X = \sum_i \frac{\mu(V_i)}{\tau_i(V_i)}\tau_i$ (with the convention that $\mu(V_i)/\tau_i(V_i)=0$ if $\tau_i(V_i)=0$) 
satisfies
\[
    \W_p(\mu,\mu_X) \leq 2h_X.
\]
In particular, if $\mu$ has compact support, then for any $N\in\N$ there exists an at most $N$-term measure $\mu_X$ satisfying
\[
    \W_p(\mu,\mu_X) = O(N^{-\frac{1}{d}}).
\]
\end{theorem}

The $N$-term bound in the second statement matches known rates for optimal quantizers in all dimensions, and matches optimal quantization rates and empirical measure approximation on compact domains when $d\geq 3$ (see \Cref{SEC:PriorArt} for details), while being faster than the empirical measure convergence rates in the cases $d=1,2$ due to the method of estimation. Moreover, these rates extend to non-compactly supported measures 
with sufficient tail decay (\Cref{PROP:noncompact_ext,COR:Noncompact}).

Our discrete approximations are based on Voronoi partitions of $\R^d$ (or a subset $\Omega$) with respect to the points $X$, and we show that the approximation rates are governed by the mesh norm of $X$. Our approximations are motivated by orthographic projection camera models from computer vision \cite{stockman2001computer} in which points are orthogonally projected onto the camera plane. An $N$-pixel grayscale image is typically considered as a matrix or vector, and is an array of $N$ pixel intensity values. However, in machine learning applications, it has been observed that treating images as vectors in Euclidean space can fail to accurately reflect the structure that appears in them. Many recent works have proposed understanding images as probability measures, for instance by mapping pixel intensities to a uniform grid in $\R^2$ \cite{cloninger2025linearized,hamm2023wassmap,khurana2022supervised,liu2025wasserstein,mathews2019molecular,negrini2024applications,wang2010optimal}. This viewpoint has been used for manifold learning and supervised classification in these references with success, as Wasserstein distances between images treated as measures are more geometrically meaningful than Euclidean distances.

\subsection{Prior Work}\label{SEC:PriorArt}

Optimal quantization of measures and empirical measure approximation have been studied in a variety of works: \cite{bucklew1982multidimensional,graf2007foundations,hardin2020asymptotic}, and \cite{canas2012learning,dereich2013constructive,fournier2015rate,singh2018minimax,weed2019sharp}, respectively. When $d>2p$, the rate of convergence of $N$-term approximations, $\mu_N$, of an absolutely continuous measure $\mu$ for both are $O(N^{-\frac1d})$.

Let $\mathcal{Q}_N:=\{\nu \in \W_p(\R^d):|\supp(\nu)|\leq N\}$ be the set of all discrete measures in $\W_p$ supported on at most $N$ points. Then the optimal quantization problem for a given measure $\mu\in \W_p(\R^d)  $ is to find a solution to
\[\quanterror(\mu,\mathcal{Q}_N)_{\W_p}:=\inf_{\nu\in\mathcal{Q}_N}\W_p(\mu,\nu).\]
Graf and Luschgy \cite[Lemma 3.1]{graf2007foundations} show that
\[\quanterror(\mu,\mathcal{Q}_N)_{\W_p} = \inf_{\substack{\alpha\subset\R^d\\ \#\alpha\leq N}}\left(\int_{\R^d}\min_i|x-\alpha_i|^p\dd\mu(x)\right)^\frac1p. \]
That is, finding an optimal quantizer (measure) is equivalent to the problem of approximating $\mu$ with $N$ centers in $\R^d$.

Bucklew and Wise \cite{bucklew1982multidimensional} prove that if $\mu\in \W_p(\R^d)$ has finite $p+\eps$ moment for some $\eps>0$ and a nontrivial absolutely continuous part, then
\[\quanterror(\mu,\mathcal{Q}_N)_{\W_p} = O(N^{-\frac1d}).\]
Interestingly, their analysis shows that the rate above only depends on the absolutely continuous part of $\mu$. Indeed, for any singular $\mu$, $\quanterror(\mu,\mathcal{Q}_N)_{\W_p} = o(N^{-\frac1d})$, which is a stronger statement. Hardin et al.~\cite{hardin2020asymptotic} merged approaches of optimal quantization and optimizing Riesz energies, which lead to further results on asymptotics of optimal quantizers.

Another large segment of the literature considers approximations of the form $\mu_N = \frac{1}{N}\sum_{i=1}^N\delta_{x_i}$ where $x_i$ are drawn i.i.d. from $\mu$. The random measure $\mu_N$ is typically called the empirical measure of $\mu$. Most works estimate $\E[\W_p(\mu,\mu_N)]$, and in contrast to optimal quantization, the bounds hold for all $N\in\N$, but sometimes require restricted assumptions on the measure $\mu$.

Fournier and Guillin \cite{fournier2015rate} show that if $\mu \in \W_p(\R^d)$ has finite $q$-th moment for some $q > p$, then for all $N\in\N$,
\begin{align*}
    \mathbb{E}[ \W_p( \mu, \mu_N )^p ] \leq C M_q^{\frac{p}{q}}(\mu) \begin{cases} N^{-\frac12} + N^{-\frac{q-p}{q}} & p > \frac{d}{2},\ q \neq 2p \\
    N^{-\frac12} \log(1+N) N^{-\frac{q-p}{q}} & p = \frac{d}{2}, \ q \neq 2p \\
    N^{-\frac{p}{d}} + N^{-\frac{q-p}{q}} & p \in (0,\frac{d}{2}), \ q \neq \frac{d}{d-p},
    \end{cases} 
\end{align*}
for some constant $C$ depending only on $p, q,$ and $d$.  This result generalized those of Dereich et al.~\cite{dereich2013constructive}. See also \cite{singh2018minimax} for related results.  

For measures on the $d$-dimensional torus, \cite{divol2021short} showed if $\mu$ is absolutely continuous with density bounded above and away from 0, then for all $N\in\N$,
\begin{align*}
    \mathbb{E}[ \W_p( \mu, \mu_N ) ] \lesssim  \begin{cases} N^{-\frac1d} & d \geq 3 \\
    N^{-\frac12} (\log(N))^{\frac12} & d = 2 \\
    N^{-\frac12} & d = 1.
    \end{cases}
\end{align*}
Ca{\~n}as and Rosasco \cite{canas2012learning} consider empirical approximation on compact smooth manifolds, and Weed and Bach \cite{weed2019sharp} prove asymptotic and finite-sample estimates in terms of clusterability and approximate low-dimensional support.

The closest work to our \Cref{THM:ApproximationTransferIntro} is that of Niles-Weed and Berthet \cite{niles2022minimax} who study nonparametric estimation of smooth 
densities in $\W_p$ from i.i.d.\ samples, establishing minimax-optimal rates for 
this problem. They show 
that wavelet estimators $\hat{f}$ based on $N$ i.i.d.\ samples from $f\in B^\alpha_{q}(\LL_{p'}(\Omega))$ (see \Cref{SEC:Johnson} for the definition of the Besov space)
achieve the minimax rate $\E\W_p(f\dd x,\hat{f}\dd x) = O(N^{-\frac{1+\alpha}{d+2\alpha}})$ for $d\geq3$, 
interpolating between the empirical measure rate $O(N^{-\frac1d})$ and the parametric rate 
$O(N^{-\frac12})$ as smoothness $\alpha$ increases. 

Additionally, Theorem 4 therein states that for $p\in[1,\infty)$ and densities $f,g\in\LL_p([0,1]^d)$ which are bounded above and away from 0 ($M\geq f(x),g(x)\geq m>0$ for a.e. $x\in[0,1]^d)$, one has
\[M^{-\frac{1}{p'}}\|f-g\|_{B^{-1}_{p.\infty}}\lesssim \W_p(f\dd x,g\dd x)\lesssim m^{-\frac{1}{p'}}\|f-g\|_{B^{-1}_{p,1}},\]
where the left and right-hand sides are norms on a negative Besov space. These type of bounds may be compared to negative Sobolev bounds like that of \cite{peyre2018comparison}.

While their work is set in the statistical estimation framework and ours in the deterministic 
approximation framework, the papers overlap in one key technical point: both use a 
Benamou--Brenier argument exploiting the lower bound $m$ on the densities to bound 
$\W_p$ in terms of a norm of $f-g$. Their approach yields the sharper Besov norm 
$\|f-g\|_{B^{-1}_{p,1}}$ on the right-hand side, but requires smoothness of the densities, and a wavelet structure on the 
domain, which in turn requires them to construct the vector field solution to the continuity equation. Our \Cref{THM:ApproximationTransfer} instead applies Bogovskii's theorem to 
obtain the bound directly in terms of $\|f-g\|_{\LL_p(\Omega)}$, which is a coarser 
quantity in principle, but holds on any bounded Lipschitz domain (even John domain), does not require that the densities be bounded above, requires no smoothness on the densities, and requires no additional structure on 
the approximation scheme. Additionally, approximation rates are often measured in the $\LL_p$ distance. Moreover, our \Cref{PROP:LinearApproximationCounterexample} 
and their Theorem 10 establish the necessity of 
the lower bound assumption for $p>1$ via different methods.

Finally, we point out works like \cite{baptista2025approximation}, which consider structured approximations not of measures themselves, but of transport maps between fixed measures.

\subsection{Outline}
\Cref{SEC:Background} discusses notation and background on function spaces, optimal transport, and Voronoi cells. \Cref{SEC:ApproximationMain} contains our main approximation results and examples of concrete approximation schemes, and \Cref{SEC:Voronoi} discusses discrete approximations on Voronoi partitions. We end with a brief conclusion.

\section{Background}\label{SEC:Background}

\subsection{Notation}

Throughout the paper, we will use $|\cdot|$ to denote Euclidean distance between vectors, cardinality of discrete sets, and Lebesgue measure of subsets of $\R^d$ where the context will make it clear which one is meant. For a subset $\Omega\subset\R^d$, $\diam(\Omega):=\sup_{x,y\in\Omega} |x-y|$ is its diameter. 

For $\Omega\subset\R^d$, $p\in[1,\infty)$ and $k\in\N$, the Sobolev space $\Sob{k}{p}(\Omega)$ is defined by finiteness of the following norm and its respective seminorm:
\[\|f\|_{\Sob{k}{p}(\Omega)}:= \|f\|_{\LL_p(\Omega)}+|f|_{\Sob{k}{p}(\Omega)},\qquad |f|_{\Sob{k}{p}(\Omega)} := \max_{|\alpha|=k}\|D^\alpha f\|_{\LL_p(\Omega)}\]
where $\alpha$ is a multi-index, and $D^\alpha$ is the corresponding differential operator.  For any $p$, we denote its H\"{o}lder conjugate as $p'$ defined by $\frac1p+\frac{1}{p'}=1$. In applying some results from PDE we assume that our domains $\Omega$ are bounded Lipschitz domains, which is defined as follows.

\begin{definition}\label{DEF:LipschitzDomain}
A nonempty, bounded, open set $\Omega\subset\R^d$ is a \emph{bounded Lipschitz domain} if for every $x_0\in\partial\Omega$ there exist $r>0$, an orthogonal transformation $R\in O(d)$, and a Lipschitz function $\gamma:\R^{d-1}\to\R$ such that, in the rotated coordinates $y = R(x-x_0)$,
\begin{align*}
\Omega\cap\BB(x_0,r) &= \{x\in\BB(x_0,r): y_d>\gamma(y_1,\dots,y_{d-1})\},\\
\partial\Omega\cap\BB(x_0,r) &= \{x\in\BB(x_0,r): y_d=\gamma(y_1,\dots,y_{d-1})\}.
\end{align*}
Equivalently, $\partial\Omega$ is locally the graph of a Lipschitz function with $\Omega$ lying on one side of the graph; see, e.g., \cite[Appendix C.1]{evans2010partial} or \cite[Definition 1.2.1.1]{grisvard1985elliptic}. When $d=1$, this reduces to the condition that $\Omega$ is a finite union of bounded open intervals.
\end{definition}

For the ball of radius $R$ centered at $0$, we use a shorthand notation $\BB_R$. We will use $C$ to denote a positive, finite constant which may change from line to line, and subscripts to denote its dependence on various parameters. We will also use $A \lesssim B$ to mean $A\leq CB$ for some constant $C$, and $A\asymp B$ to mean $B\lesssim A\lesssim B$.

\subsection{Optimal Transport}

For $\Omega\subset\R^d$ and $p\in[1,\infty)$, the Wasserstein-$p$ space, denoted $\W_p(\Omega)$ is the set of probability measures with finite $p$-th moment ($M_p(\mu):=\int_{\Omega}|x|^p\dd\mu<\infty$), equipped with the Wasserstein distance
\begin{equation}\label{EQ:Kantorovich}
    \W_p(\mu,\nu) := \inf_{\pi\in\Pi(\mu,\nu)} \left(\int_{\Omega\times\Omega}|x-y|^p\dd\pi(x,y)\right)^\frac1p,
\end{equation}
where $\mathcal{P}(\Omega\times\Omega)$ is the set of all probability measures over $\Omega\times\Omega$ and $\Pi(\mu,\nu):=\{\pi\in\mathcal{P}(\Omega\times\Omega): \pi(A\times \Omega) = \mu(A),\; \pi(\Omega\times A)=\nu(A) \textnormal{ for all Borel measurable } A\subset\Omega\}$ is the set of all joint probability measures with marginals $\mu$ and $\nu$. The problem \eqref{EQ:Kantorovich} is called the Kantorovich formulation of optimal transport, and an optimizer always exists for $\mu,\nu\in\W_p(\R^d)$, and is called the optimal transport plan \cite{villani2003topics,santambrogio2015optimal}. 

An alternate formulation of the Wasserstein distance is the Benamou--Brenier, or dynamic, formulation \cite{benamou2000computational}. For any bounded, convex $\Omega\subset\R^d$ and $p>1$, one has
\begin{equation}\label{EQN:BenamouBrenier}\W_p(\mu,\nu)^p = \inf_{\substack{\rho_t,v_t \\ \partial_t\rho_t + \nabla\cdot(\rho_t v_t) = 0}}\int_0^1\int_\Omega |v_t(x)|^p\rho_t(x) \dd x\dd t\end{equation}
where $\rho_t$ is a curve in $\mc{P}(\Omega)$ with $\rho_0 = \mu$, $\rho_1=\nu$, $v_t$ is a velocity field, and the continuity equation $\partial_t\rho_t+\nabla\cdot(\rho_t v_t) = 0$ is interpreted in the sense of distributions with Neumann boundary condition $\rho_t v_t\cdot \mathbf{n}|_{\partial\Omega} = 0$ (see, e.g., \cite[Chapters 4 and 5]{santambrogio2015optimal} for details of this formulation). In particular, we note that \cite[Theorem 5.28]{santambrogio2015optimal} does not require any regularity assumption on $v_t$ to be an admissible vector field solving the continuity equation to yield an upper bound on $\W_p(\mu,\nu)$; however, Benamou and Brenier prove that optimal solutions must be gradient vector fields of the form $v_t = \nabla\varphi_t$.

\begin{remark}\label{REM:BBUpperBound}
While Santambrogio's proof of equality in \eqref{EQN:BenamouBrenier} \cite[Theorem 5.28]{santambrogio2015optimal} requires $\Omega$ to be bounded and convex, the upper bound
\[\W_p(\mu,\nu)^p \leq \int_0^1\int_\Omega |v_t(x)|^p\dd\rho_t(x)\dd t\]
holds for any bounded domain $\Omega\subset\R^d$ and any admissible pair $(\rho_t,v_t)$ satisfying the continuity equation \eqref{EQN:ContinuityDiv} with no-flux boundary conditions. To see this, regard $\rho_t$ as a measure on $\R^d$ giving no mass to $\R^d\setminus\Omega$ and extend $v_t$ by $0$ outside $\Omega$, obtaining $(\widetilde\rho_t,\widetilde v_t)$ on $\R^d$. For any $\phi\in \CC^\infty_c(\R^d\times(0,1))$,
\[\int_0^1\int_{\R^d}\bigl(\partial_t\phi + \nabla\phi\cdot\widetilde v_t\bigr)\dd\widetilde\rho_t\dd t = \int_0^1\int_\Omega\bigl(\partial_t\phi + \nabla\phi\cdot v_t\bigr)\dd\rho_t\dd t = 0,\]
so $(\widetilde\rho_t,\widetilde v_t)$ satisfies the continuity equation on $\R^d$ in the sense of distributions. By \cite[Theorem 8.3.1]{ambrosio2008gradient}, $t\mapsto\widetilde\rho_t$ is an absolutely continuous curve in $(\W_p(\R^d),\W_p)$ whose metric derivative satisfies $|\widetilde\rho'|(t) \leq \|\widetilde v_t\|_{\LL_p(\widetilde\rho_t)}$ for a.e.\ $t\in(0,1)$. Consequently, by \cite[Theorem 1.1.2]{ambrosio2008gradient} and the fact that $\|\cdot\|_{\LL_1([0,1])}\leq\|\cdot\|_{\LL_p([0,1])}$,
\begin{align*}
\W_p(\widetilde\rho_0,\widetilde\rho_1) \leq \int_0^1 |\widetilde\rho'|(t)\dd t \leq \int_0^1 \|\widetilde v_t\|_{\LL_p(\widetilde\rho_t)}\dd t 
&\leq \left(\int_0^1\int_{\R^d} |\widetilde v_t|^p\dd\widetilde\rho_t\dd t\right)^{1/p}\\ & = \left(\int_0^1\int_\Omega |v_t|^p\dd\rho_t\dd t\right)^{1/p},
\end{align*}
where the final equality uses that $\widetilde\rho_t$ gives no mass to $\R^d\setminus\Omega$ and $\widetilde v_t = v_t$ on $\Omega$. Finally, note that since $\rho_0,\rho_1$ are supported in $\overline\Omega$, we have $\W_p(\mu,\nu) = \W_p(\rho_0,\rho_1) = \W_p(\widetilde\rho_0,\widetilde\rho_1)$ which yields the claim. Note that equality in \eqref{EQN:BenamouBrenier} may fail on non-convex $\Omega$, see \cite{figalli2010new} for a discussion.
\end{remark}

Given a measurable map $T:\R^d\to\R^d$, we denote by $T_\#\mu$ the pushforward measure which satisfies $T_\#\mu(A) = \mu(T^{-1}(A))$.

\subsection{Nonuniform points and Voronoi partitions}\label{SEC:BackgroundVoronoi}

Let $X:=\{x_i\}_{i=1}^\infty\subset\R^d$ be a set of points. We recall the notion of the \textit{mesh norm} of $X$ given by
\begin{equation}\label{EQ:MeshNorm}h_X:=\sup_{y\in\R^d}\inf_{x_i\in X}|x_i-y|.\end{equation}
The \textit{minimum separation radius} is defined by
\begin{equation}\label{EQN:SeparationRadius}q_X:=\frac12\inf_{i\neq j}|x_i-x_j|.\end{equation}
A point set $X$ is called \textit{quasi-uniform} if $0<q_X\leq h_X<\infty$.

We denote by $V_i\subset\R^d$ the Voronoi cell centered at $x_i$, defined by
\[V_i:=\{x\in\R^d:|x-x_i|\leq|x-x_j|, \textnormal{ for all } j\neq i\},\]which is the set of all points closer to $x_i$ than any other point in $X$. Clearly $\R^d = \bigcup V_i$, but the Voronoi cells can possibly overlap at the boundary. Approximating an absolutely continuous measure piecewise on a Voronoi partition as is done in \Cref{thm:piecewise-approximation-intro} does not require the Voronoi cells to be disjoint (as they overlap on a set of Lebesgue measure 0). However, to approximate a measure which has a singular part, this is an important consideration. Therefore, we note that we may modify the Voronoi cells by defining $W_1=V_1$ and $W_i:=V_i\setminus \bigcup_{j<i}W_j$, $i>1$. Then $W_i\cap W_j=\emptyset$ for all $i\neq j$, and $\R^d=\bigsqcup_i W_i$ (disjoint union). Additionally, $\diam(W_i)\leq \diam(V_i)\leq 2h_X$, which is needed in the proof in \Cref{SEC:Nonuniform}. We assume in the sequel that the Voronoi cells have been modified in this way so that they are pairwise disjoint and their union is all of $\R^d$.

The radius of the Voronoi cell centered at $x_i$ is $\rad(V_i) = \max_{x\in V_i}|x-x_i|$, and satisfies $\rad(V_i)\leq h_X$.

\section{Global approximation rates on \texorpdfstring{$\Omega$}{Omega}}\label{SEC:ApproximationMain}

The goal of this section is to prove when approximation rates transfer from a function approximation scheme to measure approximation, answering Questions \ref{Q:Smoothness} and \ref{Q:Local}.

\subsection{Transfer of approximation rates}

The following simple proposition shows that finding a pointwise upper bound of the Wasserstein distance in terms of the $\LL_p$ norm between the densities is a sufficient condition to imply both local and global approximation power transfer.

\begin{proposition}\label{PROP:ApproximationReduction}
    Let $p\in[1,\infty)$, $\Omega\subset\R^d$, and $\mc{F}, \mc{A}$ be as in \Cref{Q:Smoothness}. If there exists a constant $C<\infty$ such that \begin{equation}\label{EQN:PointwiseSmoothBound}\W_p(f\dd x,g\dd x) \leq C \|f-g\|_{\LL_p(\Omega)}\end{equation} for every $f\in\mc{F}$ and $g\in \mc{A}$, then global approximation power transfers; i.e., \eqref{EQN:ApproximationPower} holds.
    On the other hand, if for every $f\in\mc{F}$, there exists a constant $C_f$ such that \[\W_p(f\dd x,g\dd x)\leq C_f\|f-g\|_{\LL_p(\Omega)}\] for every $g\in\mc{A}$, then local approximation power transfers; i.e., \eqref{EQN:LocalApproximationPower} holds.
\end{proposition}

\begin{proof}
    By assumption, we have
    \[\mc{E}(\mc{M}_{\mc{F}},\mc{M}_{\mc{A}})_{\W_p} = \sup_{f\in\mc{F}}\inf_{g\in\mc{A}} \W_p(f\dd x,g\dd x) \leq \sup_{f\in\mc{F}}\inf_{g\in\mc{A}}C\|f-g\|_{\LL_p(\Omega)} = C\mc{E}(\mc{F},\mc{A})_{\LL_p}.\]
    The proof of the second statement follows along the same lines.
\end{proof}

We first note that we can get sublinear approximation bounds of the forms above by well-known total variation bounds with no assumptions on the densities.

\begin{theorem}\label{THM:TV}
Let $\Omega\subset\R^d$ be bounded. For any $\mc{F},\mc{A}\subset\LL_1(\Omega)$ such that $\int_\Omega f=\int_\Omega g=1$ for all $f\in\mc{F}$, $g\in\mc{A}$, both \[\mc{E}(\mc{M}_{\mc{F}},\mc{M}_{\mc{A}})_{\W_1}\leq C\mc{E}(\mc{F},\mc{A})_{\LL_1}, \quad \textnormal{and}\quad \dist(f\dd x,\mc{M}_{\mc{A}})_{\W_1}\leq C\dist(f,\mc{A})_{\LL_1}\] hold for $C=C_\Omega$ (independent of $f$ in the second bound).

When $p\in(1,\infty)$, we have
\[\mc{E}(\mc{M}_{\mc{F}},\mc{M}_{\mc{A}})_{\W_p} \leq C\mc{E}(\mc{F},\mc{A})_{\LL_p}^\frac1p,\quad \textnormal{and}\quad \dist(f\dd x,\mc{M}_{\mc{A}})_{\W_p}\leq C\dist(f,\mc{A})_{\LL_p}^\frac1p,\]
for some $C=C_{\Omega,p}$ (independent of $f$ in the second bound).
\end{theorem}

\begin{proof}
All of the bounds above follow from the fact that the $\W_p$ norm is bounded above by the total variation norm \cite[Theorem 6.15]{villani2008optimal}, which implies that \[\W_1(f\dd x,g\dd x)\leq\frac12\diam(\Omega)\|f-g\|_{\LL_1(\Omega)},\] and for $p\in(1,\infty)$, \begin{multline}\W_p(f\dd x,g\dd x)\leq 2^\frac{1}{p'}\diam(\Omega)\|f\dd x-g\dd x\|_{\TV} = 2^{\frac{1}{p'}-1}\diam(\Omega)\|f-g\|_{\LL_1(\Omega)}^\frac1p \\ \leq 2^{-\frac{1}{p}}\diam(\Omega)|\Omega|^\frac{1}{pp'}\|f-g\|_{\LL_p(\Omega)}^\frac1p.\end{multline} Appealing to \Cref{PROP:ApproximationReduction} implies the desired result.
\end{proof}

Note that \Cref{THM:TV} gives the desired transfer of approximation rates in the case $p=1$. However, without further assumptions on the densities, the above rates for $p\in(1,\infty)$, i.e., $\W_p(f\dd x,g\dd x)\leq C\|f-g\|_{\LL_p(\Omega)}^\frac1p$, are sharp, and we cannot ask for $\W_p(f\dd x,g\dd x)\leq C\|f-g\|_{\LL_p(\Omega)}$. We capture this as follows.

\begin{proposition}\label{PROP:LinearApproximationCounterexample}
Let $p\in(1,\infty)$, and $\Omega\subset\R^d$ be bounded with nonempty interior. Then there exist densities $\{f_\eps\}_{\eps>0}$, $\{g_\eps\}_{\eps>0}$ and $C>0$ such that, for all sufficiently small $\eps$, \[\W_p(f_\eps \dd x,g_\eps \dd x)\leq C\|f_\eps-g_\eps\|_{\LL_p(\Omega)}^\frac1p,\] but for which
\[\frac{\W_p(f_\eps \dd x,g_\eps \dd x)}{\|f_\eps-g_\eps\|_{\LL_p(\Omega)}^\alpha}\to\infty,\quad \eps\to0^+,\] for every $\alpha>\frac1p.$ In particular, there is no constant $C$ such that $\W_p(f_\eps \dd x,g_\eps \dd x)\leq C\|f_\eps-g_\eps\|_{\LL_p(\Omega)}$ for all sufficiently small $\eps>0$.
\end{proposition}
The proof of this proposition is given in \Cref{APP:ProofCounterexample}.

\subsection{Smooth Approximations}

Returning to \Cref{Q:Smoothness} for $p\neq1$, to achieve linear approximation rates of the form $\mc{E}(\mathcal{M}_{\mc{F}},\mc{M}_{\mathcal{A}})_{\W_p} \leq C \mc{E}(\cF,\mathcal{A})_{\LL_p}$ (i.e., without the $1/p$ power on the right-hand side), we must make more assumptions on the densities and the domain $\Omega$ involved as seen from \Cref{PROP:LinearApproximationCounterexample}. We present our main result, which has very minor restrictions on $\Omega$ and the densities.

\begin{theorem}\label{THM:ApproximationTransfer}
Let $p\in(1,\infty)$, and $\Omega\subset\R^d$ be a bounded Lipschitz domain if $d\geq2$; if $d=1$, let $\Omega=(a,b)$ for some $a<b$. Let $\mc{F},\mc{A}\subset \LL_p(\Omega)$ be such that all $f\in\mc{F}$ and $g\in\mc{A}$ satisfy $\int_\Omega f=\int_\Omega g=1$ and are bounded away from $0$; i.e., there exists an $m>0$ such that
    \[f(x),g(x) \geq m \qquad \textnormal{for a.e. } x\in\Omega.\]
    Then both \[\mc{E}(\mathcal{M}_{\mc{F}},\mc{M}_{\mathcal{A}})_{\W_p} \leq C \mc{E}(\cF,\mathcal{A})_{\LL_p}\quad \textnormal{and}\quad\dist(f\dd x,\mc{M}_{\mc{A}})_{\W_p} \leq C\dist(f,\mc{A})_{\LL_p}\] for some $C=C_{\Omega,d,p,m}$ (independent of $f$ in the second bound).
\end{theorem}

\begin{proof}
Let $f\in\mc{F}$ and $g\in\mc{A}$ be fixed, and define \[\rho_t\dd x = [(1-t)f+tg]\dd x, \qquad t\in[0,1].\] We will slightly abuse notation and use $\rho_t$ to denote the curve in $\W_p(\Omega)$ and also the density given above. Note that the assumptions on $f$ and $g$ imply that $\rho_t\geq m>0$. With this admissible choice of $\rho_t$, we look for a solution $E_t$ to the continuity equation in the Benamou--Brenier formulation of the form
\begin{equation}\label{EQN:ContinuityDiv}
\begin{cases}\nabla\cdot E_t = f-g, & \textnormal{in  } \Omega,\\ E_t\cdot\mathbf{n}=0,& \textnormal{on  } \partial\Omega.\end{cases}\end{equation}
Given a solution to \eqref{EQN:ContinuityDiv}, we set $v_t = E_t/\rho_t$, which is well-defined since $\rho_t\geq m>0$ almost everywhere.

For ambient dimension $d\geq 2$, Bogovskii \cite{bogovskii1979solution} (see \cite[Chapter III]{galdi2011introduction} for the formulation used here) proves that for bounded Lipschitz $\Omega\subset\R^d$ and $h\in\LL_p(\Omega)$ with $\int_\Omega h=0$, there exists a distributional solution $E_t\in \Sob{1}{p}_0(\Omega)$ to
\[\begin{cases} -\nabla\cdot E_t = h, & \textnormal{in } \Omega,\\
E_t\cdot\mathbf{n} = 0, & \textnormal{on } \partial\Omega\end{cases}
\]
that satisfies $|E_t|_{\Sob{1}{p}_0(\Omega)}\leq C_{\Omega,d,p}\|h\|_{\LL_p(\Omega)}$, which implies, on account of the Poincar\'{e} inequality,
\begin{equation}\label{EQN:BogovskiiBound}
    \|E_t\|_{\Sob{1}{p}(\Omega)} \leq C\|h\|_{\LL_p(\Omega)}.
\end{equation}
We apply this result to \eqref{EQN:ContinuityDiv} noting that $\int_\Omega (f-g) = 0$ since both are densities, and set $v_t = E_t/\rho_t$. By \Cref{REM:BBUpperBound}, \eqref{EQN:BenamouBrenier} provides an upper bound for $\W_p(f\dd x,g\dd x)^p$, so applying \eqref{EQN:BogovskiiBound} and $\rho_t\geq m$ yields
\begin{align*}\W_p(f\dd x,g\dd x)^p & \leq \int_0^1\int_\Omega|v_t(x)|^p\rho_t(x)\dd x\dd t\\
& = \int_0^1\int_\Omega |E_t(x)|^p\rho_t(x)^{1-p}\dd x\dd t\\
& \leq m^{1-p}\|E_t\|_{\LL_p(\Omega)}^p\\
& \leq Cm^{1-p}\|f-g\|_{\LL_p(\Omega)}^p,
\end{align*}
whence taking $p$-th roots and relabeling the constant $C$ yields \eqref{EQN:PointwiseSmoothBound}, and thus \eqref{EQN:ApproximationPower} and \eqref{EQN:LocalApproximationPower}, on account of \Cref{PROP:ApproximationReduction}.

The case $d=1$ is simpler: since $\Omega=(a,b)$, the continuity equation becomes $E_t' = f-g$ on $(a,b)$, and $E_t(x) := \int_a^x(f-g)\dd y$ satisfies the boundary conditions $E_t(a) = 0$ and $E_t(b) = \int_a^b(f-g)\dd y =0$ (since $f,g$ are densities). Thus $E_t\in \Sob{1}{p}_0((a,b))$, and elementary estimates give $\|E_t\|_{\LL_p((a,b))}\leq (b-a)\|E_t'\|_{\LL_p((a,b))} = (b-a)\|f-g\|_{\LL_p((a,b))}$. The remainder of the proof follows along the same lines as above.
\end{proof}

\begin{remark}\label{REM:mDependence}
Note that the constants in the upper bounds in \Cref{THM:ApproximationTransfer} are of the form $C_{\Omega,d,p}m^{\frac1p-1}= C_{\Omega,d,p}m^{-\frac{1}{p'}}$ where $p'$ is the H\"{o}lder conjugate of $p$. This explicit form for the reliance on $m$ will be important for applications of this theorem in the sequel.
\end{remark}

The assumptions of \Cref{THM:ApproximationTransfer} are essentially minimal, at least as far as utilizing \Cref{PROP:ApproximationReduction} to draw the conclusion. If $\Omega$ is unbounded, then one can simply take $f = \one_Q$ where $Q = [0,1]^d$, and $g = f(\cdot-t)$; then $\W_p(f\dd x,g\dd x) = |t|\to\infty$, whereas $\|f-g\|_{\LL_p} = 2^\frac1p$, so \eqref{EQN:PointwiseSmoothBound} cannot hold. \Cref{PROP:LinearApproximationCounterexample} shows the necessity of $f$ and $g$ being bounded away from $0$. Bogovskii's theorem allows for arbitrary $\LL_p(\Omega)$ densities, whereas elliptic PDE techniques would require some smoothness (e.g., $\CC^1$ \cite[Chapter 9]{gilbarg1977elliptic}). One can relax the assumption that $\Omega$ is a Lipschitz domain. While Bogovskii's original theorem \cite{bogovskii1979solution} is stated for finite sums of domains that are starlike with respect to a ball, Acosta, Dur\'an, and Muschietti \cite{acosta2006solutions} established that the divergence equation admits a solution with the $\Sob{1}{p}_0$ bound \eqref{EQN:BogovskiiBound} on any bounded John domain, which is a strictly larger class than bounded Lipschitz domains (whose boundaries may be much worse than Lipschitz). Consequently, \Cref{THM:ApproximationTransfer} remains valid for any bounded John domain $\Omega$.

\subsection{Lower bounds}\label{SEC:LowerBound}

Here we consider a lower bound on $\W_p(f\dd x, g\dd x)$ in terms of a negative Sobolev norm. \Cref{THM:LowerBound} extends negative Sobolev bounds of Peyre \cite{peyre2018comparison} (see also \cite{maury2010macroscopic}) to $p\neq2$, requiring a modified formulation of the negative Sobolev spaces. The construction of the spaces involved is slightly different from the usual formulation of negative Sobolev space (see, e.g., \cite{adams2003sobolev}), although it follows similar lines. The spaces formed below account for zero-mean functions, which are natural to consider in the setting of differences of densities. The proof technique here is known and follows the lines of \cite{maury2010macroscopic,niles2022minimax}; our formulation is to make concrete the subspace of the Sobolev space involved.

\begin{definition}
Let $p\in(1,\infty)$ and $\Omega\subset\R^d$ be bounded and convex. Let
\[
\Sob{1}{p}_{\diamond}(\Omega)
:=
\left\{
h\in \Sob{1}{p}(\Omega): \int_\Omega h\dd x=0
\right\},
\]
which we equip with the norm $\|h\|_{\Sob{1}{p}_\diamond(\Omega)}:=\|\nabla h\|_{\LL_p(\Omega)}.$
Let $\Sob{-1}{p}_\diamond(\Omega)$ be the dual space $(\Sob{1}{p'}_\diamond(\Omega))^*$ equipped with the usual dual norm.
\end{definition}

Note that the gradient norm defined above on $\Sob{1}{p}_\diamond(\Omega)$ is equivalent to the usual Sobolev norm when restricted to the set of zero-mean functions on account of the Poincar\'{e} inequality. In particular,
$\Sob{1}{p}_\diamond(\Omega)$ is a Banach space, and therefore
$\Sob{-1}{p}_\diamond(\Omega)$ is a Banach space. 

For mean-zero $h\in\LL_p(\Omega)$, we identify it with its functional $T_h\in \Sob{-1}{p}_\diamond(\Omega)$ via $T_h(\phi):=\int_\Omega h\phi$ for $\phi\in \Sob{1}{p'}_\diamond(\Omega)$, and with a slight abuse of notation write
\begin{equation}\label{EQ:NegativeSobolevNorm}
\|h\|_{\Sob{-1}{p}_\diamond(\Omega)}
:=
\sup\left\{
\left|\int_\Omega h\phi\dd x\right|:
\phi\in \Sob{1}{p'}_\diamond(\Omega),\ 
\|\nabla \phi\|_{\LL_{p'}(\Omega)}\leq 1
\right\}.
\end{equation}

\begin{remark}
The space $\Sob{-1}{p}_\diamond(\Omega)$ should be understood as a homogeneous
negative Sobolev space. It is the natural analogue, for general $p$, of the
space $\dot H^{-1}(\Omega)$ obtained by dualizing the space of zero-mean
$H^1$ functions endowed with the norm $\|\nabla\cdot\|_{\LL_2(\Omega)}$ (see \cite[Section 5.5.2]{santambrogio2015optimal} for a discussion).

This is different from the standard negative Sobolev space
$\Sob{-1}{p}(\Omega):=(\Sob{1}{p'}_0(\Omega))^*$, where the test functions have zero trace, since test functions in \eqref{EQ:NegativeSobolevNorm} instead have zero-mean. This formulation is well suited to optimal transport estimates
since differences of probability densities have zero-mean, and this formulation is the natural equivalent to that of $\dot{H}^{-1}(\Omega)$ studied by Peyre \cite{peyre2018comparison}.
\end{remark}

\begin{proposition}\label{THM:LowerBound}
Let $p\in(1,\infty)$, and let $\Omega\subset\R^d$ be bounded and convex. Let
$f,g\in\LL_p(\Omega)$ satisfy $\int_\Omega f=\int_\Omega g=1$ and $0\leq
f(x),g(x)\leq M<\infty$ for a.e.\ $x\in\Omega$. Then
\begin{equation}\label{EQN:LowerBound}
M^{-\frac{1}{p'}}\|f-g\|_{\Sob{-1}{p}_\diamond(\Omega)} \leq \W_p(f\dd x,g\dd x).
\end{equation}
If additionally $f(x),g(x)\geq m>0$ for a.e. $x\in\Omega$, then we have
\[M^{-\frac{1}{p'}}\|f-g\|_{\Sob{-1}{p}_\diamond(\Omega)} \leq \W_p(f\dd x,g\dd x) \leq C_{\Omega,d,p}m^{-\frac{1}{p'}}\|f-g\|_{\LL_p(\Omega)}.\]
\end{proposition}

\begin{proof}
The upper bound in the second part comes directly from \Cref{THM:ApproximationTransfer,REM:mDependence}. To see the lower bound, let $\phi\in\Sob{1}{p'}_\diamond(\Omega)$ with $\|\nabla \phi\|_{\LL_{p'}(\Omega)}\leq1,$ and note that by \cite[Lemma 3.4 and ensuing remark]{maury2010macroscopic},

\[
\left|\int_\Omega\phi(f-g)\dd x\right|\leq M^{\frac{1}{p'}}\W_p(f\dd x,g\dd x).
\]
Taking supremum over all admissible $\phi$ gives the desired bound on account of \eqref{EQ:NegativeSobolevNorm}.

\end{proof}

Niles--Weed and Berthet prove a similar lower bound to \eqref{EQN:LowerBound} with a negative Besov space norm on the left-hand side in the case $\Omega=[0,1]^d$ (\!\!\cite[Theorem 4]{niles2022minimax}). Their bound also contains an additional constant factor depending on the wavelet construction of the Besov space. The following proposition shows that our technique results in a bound that is bounded below by theirs up to a constant factor. The negative Besov space there is denoted by $B_{p,\infty}^{-1}(\Omega)$, and can be defined to be the dual (\!\cite[Theorem 7.1.7]{triebel1973spaces}) of what we denote by $B_1^1(\LL_{p'}(\Omega))$ in \Cref{SEC:Johnson} following the convention of DeVore and Sharpley \cite{devore1993besov}. The precise definitions of the spaces are unimportant at present, so we defer them to the sequel.

\begin{proposition}\label{PROP:BesovSobolevEmbed}
Let $p\in(1,\infty)$, and let $\Omega\subset\R^d$ be a bounded Lipschitz domain.
If $h\in\LL_p(\Omega)$ with $\int_\Omega h\dd x=0$, then
\[
\|h\|_{B^{-1}_{p,\infty}(\Omega)}\leq C_{\Omega,d,p}\|h\|_{\Sob{-1}{p}_\diamond(\Omega)}.
\]
\end{proposition}
\begin{proof}
    See \Cref{APP:PropBesovProof}.
\end{proof}

\subsection{Approximation Scheme Examples}\label{SEC:Examples}

Here we provide some concrete approximation schemes that satisfy the conditions of \Cref{THM:ApproximationTransfer}. The main difficulty is that much of function approximation theory does not enforce that the approximant have the same integral as the target function. Additionally, making the approximation scheme positivity-preserving is often not required, although it has been studied \cite{de2017positive,nochetto2002positivity}.  We begin by noting that one can modify an approximation to an density $f\in\LL_p(\Omega)$ by shifting and renormalizing it in a natural way, and we find that the error incurred by shifting and normalizing can be bounded by a constant factor of the $\LL_p$ norm of the target function. 

\begin{theorem}\label{THM:clipped-positivity}
Let $p\in(1,\infty)$ and $\Omega\subset\R^d$ be a bounded Lipschitz domain.
Let $f\in \LL_p(\Omega)$ satisfy $\int_\Omega f=1$ and $f(x)\geq m > 0$ for a.e. $x\in\Omega.$
Let $A:\LL_p(\Omega)\to \LL_p(\Omega)$ be an approximation scheme with $\|A(f)\|_{\LL_p(\Omega)}\leq C_A\|f\|_{\LL_p(\Omega)}$, and define
\[
S_A(f)(x) := \max\{A(f)(x), m\}, 
\qquad
Z := \int_\Omega S_A(f)(x)\dd x,
\qquad
\widetilde{A}(f) := Z^{-1} S_A(f).
\]
Then there exists a constant $C_{\Omega,d,p,m,A}$ such that
\[
\W_p(f\dd x, \widetilde{A}(f)\dd x)
   \leq C_{\Omega,d,p,m,A}\|f\|_{\LL_p(\Omega)}^{1+\frac{1}{p'}}
      \|f-A(f)\|_{\LL_p(\Omega)}.
\]
\end{theorem}

\begin{proof}
First, note that almost everywhere, $|f-S_A(f)|\leq |f-A(f)|$ since $S_A(f)$ is formed by clipping $A(f)$ when it is below the lower bound on $f$. Therefore
\begin{equation}\label{eq:clip-Lp}
\|f - S_A(f)\|_{\LL_p(\Omega)} \leq \|f - A(f)\|_{\LL_p(\Omega)}.
\end{equation} 
By definition, $Z = \int_\Omega S_A(f)\geq m|\Omega|$, and by H\"older's inequality and \eqref{eq:clip-Lp},
\begin{equation}\label{eq:clip-Z-diff}
|Z - 1| = \left|\int_\Omega (S_A(f) - f)\right|
  \leq \|f - S_A(f)\|_{\LL_1(\Omega)}
  \leq |\Omega|^{1-\frac{1}{p}}\|f - A(f)\|_{\LL_p(\Omega)}.
\end{equation}
Combining these yields
\begin{equation}\label{eq:clip-Zinv}
|1 - Z^{-1}| = \frac{|Z - 1|}{Z}
  \leq \frac{|\Omega|^{-\frac{1}{p}}}{m}\|f - A(f)\|_{\LL_p(\Omega)}.
\end{equation}
Additionally, from \eqref{eq:clip-Lp}, we have
\begin{equation}\label{eq:clip-S-Lp}
\|S_A(f)\|_{\LL_p(\Omega)} \leq \|f\|_{\LL_p(\Omega)} + \|f - A(f)\|_{\LL_p(\Omega)}.
\end{equation}

Now write
\[
f - \widetilde{A}(f) = (f - S_A(f)) + (1 - Z^{-1})S_A(f).
\]
By the triangle inequality and \eqref{eq:clip-Lp}--\eqref{eq:clip-S-Lp},
\begin{align}
\|f - \widetilde{A}(f)\|_{\LL_p(\Omega)}
  &\leq \|f - S_A(f)\|_{\LL_p(\Omega)}
       + |1 - Z^{-1}|\|S_A(f)\|_{\LL_p(\Omega)} \nonumber\\
  &\leq \|f - A(f)\|_{\LL_p(\Omega)}
       + \frac{|\Omega|^{-\frac{1}{p}}}{m}\|f - A(f)\|_{\LL_p(\Omega)}
         \big(\|f\|_{\LL_p(\Omega)} + \|f - A(f)\|_{\LL_p(\Omega)}\big) \nonumber\\
  &= \left(1 + \frac{|\Omega|^{-\frac{1}{p}}}{m}
     \|f\|_{\LL_p(\Omega)}\right)\|f - A(f)\|_{\LL_p(\Omega)}
     + \frac{|\Omega|^{-\frac{1}{p}}}{m}\|f - A(f)\|_{\LL_p(\Omega)}^2.
     \label{eq:clip-Lp-final}
\end{align}
Since $f \geq m$ on $\Omega$, we have $\|f\|_{\LL_p(\Omega)} \geq m|\Omega|^{\frac{1}{p}}$,
whence $1 \leq (m|\Omega|^{\frac{1}{p}})^{-1}\|f\|_{\LL_p(\Omega)}$. Substituting
into~\eqref{eq:clip-Lp-final} and collecting terms yields
\begin{equation}\label{eq:clip-Lp-final2}
\|f - \widetilde{A}(f)\|_{\LL_p(\Omega)}
  \leq \frac{2|\Omega|^{-\frac1p}}{m}\|f\|_{\LL_p(\Omega)}\|f - A(f)\|_{\LL_p(\Omega)}
     + \frac{|\Omega|^{-\frac1p}}{m}\|f - A(f)\|_{\LL_p(\Omega)}^2.
\end{equation}
Next, we note that 
$\|f-A(f)\|_{\LL_p(\Omega)} \leq (1+C_A)\|f\|_{\LL_p(\Omega)}$, which implies that the second term of \eqref{eq:clip-Lp-final2} can be absorbed into the first term giving
\begin{equation}\label{eq:clip-Lp-final3}
\|f - \widetilde{A}(f)\|_{\LL_p(\Omega)}
  \leq C\|f\|_{\LL_p(\Omega)}\|f - A(f)\|_{\LL_p(\Omega)},\end{equation}
  for a constant $C$ depending on $\Omega, d, p,$ and $A.$

Finally, we verify the hypotheses of \Cref{THM:ApproximationTransfer}. Both $f$ and
$\widetilde{A}(f)$ lie in $\LL_p(\Omega)$ with $\int_\Omega f = \int_\Omega \widetilde{A}(f) = 1$.
Moreover, $f \geq m > 0$ a.e., and
\begin{align}\label{EQN:ZUpperBound}
Z  = \|S_A(f)\|_{\LL_1(\Omega)} & \leq \|A(f)\|_{\LL_1(\Omega)} + m|\Omega| \nonumber\\ & \leq |\Omega|^{1-\frac1p}\|A(f)\|_{\LL_p(\Omega)}+m|\Omega|\nonumber \\ & \leq C_A|\Omega|^{1-\frac1p}\|f\|_{\LL_p(\Omega)}+m|\Omega|\end{align}
which is finite, hence
\[
\widetilde{A}(f)(x) = Z^{-1}S_A(f)(x) \geq \frac{m}{Z} > 0 \quad \text{a.e.}
\]
so we can apply \Cref{THM:ApproximationTransfer} with $\mc{F}=\{f\}$ and $\mc{A}=\{\widetilde{A}(f)\}$ with the common lower bound
$\widetilde{m} := \min\{m, m/Z\}>0$ to get
\[
\W_p(f\dd x, \widetilde{A}(f)\dd x)
  \leq C\|f - \widetilde{A}(f)\|_{\LL_p(\Omega)},
\]
where $C = C_{\Omega, p, \widetilde{m},A}$. From \Cref{THM:ApproximationTransfer}, the dependence on $\widetilde{m}$ in $C_{\Omega,d,p,\widetilde{m},A}$ is of the order $m^{-\frac{1}{p}}Z^{\frac{1}{p'}}$, hence has a dependence on $f$ of $\|f\|_{\LL_p(\Omega)}^{\frac{1}{p'}}$ by \eqref{EQN:ZUpperBound}. Combining this with \eqref{eq:clip-Lp-final3} yields
\[
\W_p(f\dd x, \widetilde{A}(f)\dd x)
  \leq C_{\Omega,d,p,A}\|f\|_{\LL_p(\Omega)}^{1+\frac{1}{p'}}
     \|f - A(f)\|_{\LL_p(\Omega)},
\]
which is the claimed estimate.

\end{proof}

The case $p=1$ can be stated somewhat differently, as we need not have $f$ bounded away from 0 and $\Omega$ need not have Lipschitz boundary on account of \Cref{THM:TV}.

\begin{theorem}\label{THM:L1Clip}
Let $\Omega\subset\R^d$ be bounded. Let $f\in\LL_1(\Omega)$ satisfy $\int_\Omega f=1$
and $f(x)\geq 0$ for a.e. $x\in\Omega$. Let $A:\LL_1(\Omega)\to\LL_1(\Omega)$ be an approximation scheme such that $\|f-A(f)\|_{\LL_1(\Omega)}<1$.
Define
\[
S_A(f)(x) := \max\{A(f)(x),0\},\quad Z := \int_\Omega S_A(f)(x)\dd x,
\quad \widetilde{A}(f) := Z^{-1}S_A(f).
\]
Then there exists a constant $C_{\Omega}$ such that
\[
\W_1(f\dd x,\widetilde{A}(f)\dd x) 
  \leq C_{\Omega}\|f-A(f)\|_{\LL_1(\Omega)}.
\]
\end{theorem}
\begin{proof}
    The proof follows along the lines of the proof of \Cref{THM:clipped-positivity}, but since $f$ is not bounded away from $0$, the assumption $\|f-A(f)\|_{\LL_1(\Omega)}<1$ implies
    \[Z \geq \int_\Omega f-\|f-S_A(f)\|_{\LL_1(\Omega)} \geq 1-\|f-A(f)\|_{\LL_1(\Omega)}>0.\]
    Thus $\widetilde{A}(f)$ is well defined. Now write $f-\widetilde A(f) = \bigl(f-S_A(f)\bigr) + \bigl(S_A(f)-\widetilde A(f)\bigr)$ and notice that $\|S_A(f)\|_{\LL_1(\Omega)}=Z$ since $S_A(f)\geq0$. Thus we have
\[\|S_A(f)-\widetilde A(f)\|_{\LL_1(\Omega)} = |1-Z^{-1}|\,\|S_A(f)\|_{\LL_1(\Omega)} = |1-Z^{-1}|\,Z = |Z-1|.\]
Since $|Z-1| = \bigl|\int_\Omega(S_A(f)-f)\bigr| \leq \|S_A(f)-f\|_{\LL_1(\Omega)} \leq \|f-A(f)\|_{\LL_1(\Omega)}$,
\[\|f-\widetilde A(f)\|_{\LL_1(\Omega)} \leq \|f-S_A(f)\|_{\LL_1(\Omega)} + |Z-1| \leq 2\|f-A(f)\|_{\LL_1(\Omega)}.\]
Applying \Cref{THM:TV} gives the desired conclusion.

\end{proof}

While the former theorems actually modify the approximation space by clipping the approximation, if one can guarantee good $\LL_\infty$ approximation, then one can guarantee the same approximation rate as above without modifying the approximation scheme beyond normalization. 

\begin{theorem}\label{thm:positivity-no-clipping}
Let $p\in(1,\infty)$, and $\Omega\subset\R^d$ be a bounded Lipschitz domain. Let 
$f\in \LL_p(\Omega)$ satisfy $\int_\Omega f = 1$ and $f(x)\ge m>0$ for a.e.
$x\in\Omega$. Let $A:\LL_p(\Omega)\to\LL_p(\Omega)$ be an approximation 
scheme such that
\[
\|f-A(f)\|_{\LL_\infty(\Omega)} \leq \frac{m}{2}.
\]
Define
\[
Z = \int_\Omega A(f)(x)\dd x,
\quad\textnormal{and}\quad
\widetilde{A}(f) = Z^{-1} A(f).
\]
Then there exists a constant $C_{\Omega,d,p,m}>0$ such that
\[
\W_p(f\dd x, \widetilde{A}(f)\dd x)
   \le C_{\Omega,d,p,m}\|f\|_{\LL_p(\Omega)}
      \|f-A(f)\|_{\LL_p(\Omega)}.
\]
If $p=1$, then the assumption that $\Omega$ is Lipschitz can be removed, and $\|f\|_{\LL_1(\Omega)}=1$ in the final bound.
\end{theorem}

\begin{proof}
The hypothesis $\|f-A(f)\|_{\LL_\infty(\Omega)} \le m/2$ implies
\[
A(f)(x) \ge f(x) - \|f-A(f)\|_{\LL_\infty(\Omega)} \ge m - \frac{m}{2} = \frac{m}{2}
\quad\text{for a.e. } x\in\Omega,
\]
so $A(f)$ is strictly positive. Since $A(f)\ge m/2$,
\begin{equation}\label{eq:Z-lower-no-clip}
Z = \int_\Omega A(f)(x)\dd x \ge \frac{m}{2}|\Omega|.
\end{equation}
Note that similar to $\eqref{eq:clip-Z-diff}$ 
\[Z\leq 1+|\Omega|\|f-A(f)\|_{\LL_\infty(\Omega)}\leq 1+\frac{m|\Omega|}{2}.\]
Consequently, 
\begin{equation}\label{EQN:AtildeLinftyBound}
\widetilde{A}(f)(x) \geq \frac{m}{2Z}\geq \frac{m/2}{1+m|\Omega|/2}=:\widetilde{m}.
\end{equation}
The remainder of the proof follows the same reasoning as \Cref{THM:clipped-positivity} \textit{mutatis mutandis} but with \eqref{EQN:AtildeLinftyBound} implying that the dependence on $\widetilde{m}$ results only in extra dependence on $m$ and $\Omega$, but not $\|f\|_{\LL_p(\Omega)}$ or the operator $A$ (other than the assumption that $\|f-A(f)\|_{\LL_\infty(\Omega)}\leq \frac{m}{2}$).
\end{proof}

The remainder of this section describes some examples of concrete approximation schemes whose known rates of approximation on functions of various smoothness can be transferred to measure approximation via the above results.

\subsubsection{Taylor Polynomial Approximation}
Recall that if $f\in \CC^{k+1}(\Omega)$, then Taylor's remainder theorem states that for $x_0\in\Omega$, the Taylor polynomial \[T_kf(x):= \sum_{|\alpha|\leq k}\frac{D^\alpha f(x_0)}{\alpha!}(x-x_0)^\alpha \in\CC^{k+1}(\Omega)\]
satisfies 
\begin{equation}\label{EQN:TaylorRemainder}
    \|f-T_kf\|_{\LL_\infty(\Omega)} \leq \frac{\diam(\Omega)^{k+1}}{(k+1)!}\max_{|\alpha|=k+1}\|D^\alpha f\|_{\LL_\infty(\Omega)}. 
\end{equation}

\begin{corollary}\label{COR:TaylorExample}
    With the notations and assumptions of \Cref{thm:positivity-no-clipping}, suppose $f\in\CC^{k+1}(\Omega)$ and $f(x)\geq m>0$ with $\frac{\diam(\Omega)^{k+1}}{(k+1)!}\max_{|\alpha|=k+1}\|D^\alpha f\|_{\LL_\infty(\Omega)}\leq \frac{m}{2}.$ Then defining $Z := \int_\Omega T_kf$ and $\widetilde{T}_kf:= Z^{-1}T_kf$ yields the following for all $p\in[1,\infty)$:
    \[    \W_p(f\dd x,\widetilde{T}_kf\dd x) \leq C_{f,\Omega,p}\frac{|\Omega|^{k+1}}{(k+1)!}.\]
\end{corollary}
\begin{proof}
The case $p=1$ follows from the TV bound. To see the case $p\neq1$, apply \Cref{thm:positivity-no-clipping} with \eqref{EQN:TaylorRemainder}. 
\end{proof}

\subsubsection{Nonstationary Shift-invariant Space Approximation}\label{SEC:Johnson}

Shift-invariant spaces have been studied in approximation theory for some time, especially in the context of radial basis function (RBF) approximation, e.g., \cite{de1994approximation,aldroubi2001nonuniform,jia1997shift,buhmann2000radial,johnson1997approximation} (this is far from an exhaustive list, but comprises many of the classical results in the area). Given a generator $\phi:\R^d\to\R$, the corresponding \textit{principal shift-invariant space} in $\LL_p(\R^d)$ is defined by
\[S_p(\phi):=\left\{\sum_{j\in\Z^d}c_j\phi(\cdot-j):(c_j)\in\ell_p\right\}.\]
Note that some works define the space by $\overline{\sspan}^{\LL_p}\{\phi(\cdot-j):j\in\Z^d\},$ but for smoothly decaying generators, the definitions coincide.

For $h>0$, we define $S_p^h(\phi):= \{f(\cdot/h):f\in S_p(\phi)\}$, and the collection of shift-invariant spaces $\{S^h(\phi)\}_{h>0}$ is called the \textit{stationary ladder} of shift-invariant spaces with the generator $\phi$. Note that the generator in this case is independent of the scale parameter $h$. Global approximation rates in shift-invariant may be characterized in terms of the Strang--Fix conditions and the structure of the Fourier transform of the generator \cite{jia1997shift,kyriazis1995approximation}. 

A different phenomenon appears in \emph{nonstationary ladders}, where the 
generator $\phi_h$ itself depends on $h$. The family 
$\{S_p^h(\phi_h)\}_{h>0}$ may then achieve approximation rates unattainable by any 
fixed kernel. Johnson's proof of this for nonstationary Gaussian ladders~\cite{johnson1997approximation} is a 
standard example. To utilize his result, we must first define Besov spaces of fractionally smooth functions. 

The Besov space $B_1^{\alpha}(\LL_p(\Omega))$ can be defined as follows \cite{devore1993besov}. Let $\Omega\subset\R^d$ be open, $r\in\N$, and $y\in\R^d$. Let $\tau_y:\LL_p\to\LL_p$ be the translation operator given by $\tau_y f := f(\cdot+y)$, $\id$ be the identity operator, and $\Delta_y^r:=(\tau_y-\id)^r$ be the difference operator of order $r$, and define 
\[\Delta_y^r(f,x,\Omega):= \begin{cases} \Delta_y^r(f,x), & x, x+y, \dots,x+ry\in\Omega,\\ 0, & \textnormal{otherwise.}\end{cases}\]
Then the modulus of smoothness of $f\in\LL_p(\Omega)$ of order $r$ is
\[\omega_r(f,t,\Omega)_p:= \sup_{|y|\leq t}\|\Delta_y^r(f,\ \cdot\ ,\Omega)\|_{\LL_p(\Omega)}.\]
For $r>\alpha>0$, the Besov space is defined by finiteness of the following quasi-norm:
\[\|f\|_{B_1^{\alpha}(\LL_p(\Omega))}:= \int_0^1t^{-\alpha-1}\omega_r(f,t,\Omega)_p \dd t + \|f\|_{\LL_p(\Omega)}.\]
Note that different $r>\alpha$ values give equivalent norms, and there are various different equivalent definitions of Besov spaces \cite{sawano2018theory}.

Two important notes for our purposes are that if $\Omega$ is bounded, then $\|f\|_{B_1^{\alpha}(\LL_p(\Omega))} \leq |\Omega|^\frac1p \|f\|_{B_1^{\alpha}(\LL_\infty(\Omega))}$ by H\"{o}lder's inequality; that is, $B_1^\alpha(\LL_\infty(\Omega))$ continuously embeds into $B_1^\alpha(\LL_p(\Omega))$. Second, DeVore and Sharpley \cite{devore1984maximal}
show that for bounded Lipschitz domains, for every $p\in(1,\infty)$ $f\in B_1^{\alpha}(\LL_p(\Omega))$, there exists an extension $E(f)\in B_1^{\alpha}(\LL_p(\R^d))$ such that $E(f)|_\Omega = f$, and \begin{equation}\label{EQN:BesovExtension}
\|E(f)\|_{B_1^{\alpha}(\LL_p(\R^d))}\leq C_E\|f\|_{B_1^{\alpha}(\LL_p(\Omega))}.
\end{equation}

To describe Johnson's result, for $h\in(0,1]$ and fixed smoothness parameter $\alpha>0$, we define the Gaussian generator in the Fourier domain as
\[\widehat{\phi_h}(\xi) := e^{-\alpha(1-\log(h))|\xi|^2}\]
with the convention $\widehat{f}(\xi):=\int_{\R^d}f(x)e^{-2\pi i \bracket{x,\xi}}\dd x$.

Johnson~\cite{johnson1997approximation} shows that, for all $p\in[1,\infty]$, there exists a constant $C$ such that
\begin{equation}\label{EQN:Johnson}
\inf_{s_h \in S_p^h(\phi_h)}
\|f - s_h\|_{\LL_p(\R^d)}
\leq
Ch^{\alpha}\|f\|_{B_{1}^\alpha(\LL_p(\R^d))},
\qquad
h\in(0,1].
\end{equation}
Since \Cref{THM:ApproximationTransfer} requires that functions be bounded away from 0 on $\Omega$, but Johnson's result holds for fully supported functions in the shift-invariant ladder, to state our result about densities, we first consider densities in $B_1^\alpha(\LL_p(\Omega))$, which we then extend to all of $\R^d$, form the approximant in $S_p^h(\phi_h)$, and then restrict the approximant to $\Omega$. Since $S_p^h(\phi_h)$ is invariant to multiplication by constants, normalizing a given $s_h$ by its integral over $\Omega$ is no problem. For simplicity, we write $S_p^h(\phi_h)|_\Omega := \{f|_\Omega: f\in S_p^h(\phi_h)\}$, and write $S_p^h(\phi_h)|_\Omega \dd x$ to denote an element of $\mc{P}(\Omega)$ with a density in $S_p^h(\phi_h)|_\Omega$.

\begin{corollary}\label{COR:JohnsonExample}
    Take the notations and assumptions of \cref{THM:ApproximationTransfer,thm:positivity-no-clipping}, and let $p\in[1,\infty)$, $\alpha>0$. For every $f\in B_1^\alpha(\LL_\infty(\Omega))$ with $\int_\Omega f=1$ and $f\geq m>0$ for a.e. $x\in\Omega$,
    \[\dist(f\dd x, S_p^h(\phi_h)|_\Omega \dd x)_{\W_p(\Omega)} = O(\dist(f,S_p^h(\phi_h)|_\Omega)_{\LL_p(\Omega)}) = O(h^\alpha), \quad h\to0,\]
    where the implied constant in the first big-$O$ is $C_{\Omega,d,p,m}\|f\|_{\LL_p(\Omega)}$, and the second contains a factor of $\|f\|_{B_1^\alpha(\LL_p(\Omega))}$.
\end{corollary}
\begin{proof}
The case $p=1$ follows from the TV bound, so we give the proof for $p\neq1$. We first establish the second big-$O$ statement. Let $E(f)$ be the extension of $f$ to $B_1^\alpha(\LL_p(\R^d))$, and let $\chi:\R^d\to\R^d$ be a smooth cutoff function such that $\chi(x)=1$ on $\Omega$, and $\chi(x) = 0$ outside some neighborhood of $\Omega$. Then $\chi E(f)\in B_1^\alpha(\LL_p(\R^d))$ and
\begin{equation}\label{EQN:BesovExtensionTruncated}
    \|\chi E(f)\|_{B_1^\alpha(\LL_p(\R^d))} \leq C_\chi\|E(f)\|_{B_1^\alpha(\LL_p(\R^d))} \leq C_\chi C_E\|f\|_{B_1^\alpha(\LL_p(\Omega))}
\end{equation}
on account of \cite[Theorem 2.28]{Triebel2020Theory}.

Now for any $s_h\in S_p^h(\phi_h)$,
\[\|f-s_h|_\Omega\|_{\LL_p(\Omega)} = \|E(f)-s_h|_\Omega\|_{\LL_p(\Omega)} = \|\chi E(f)-s_h|_\Omega\|_{\LL_p(\Omega)} \leq \|\chi E(f)-s_h\|_{\LL_p(\R^d)}.\]
Since $\Omega$ is bounded, $B_1^\alpha(\LL_\infty(\Omega))\subset B_1^\alpha(\LL_p(\Omega))$, so the extension is well-defined. Taking infimum over $s_h$ and applying Johnson's result \eqref{EQN:Johnson} and \eqref{EQN:BesovExtensionTruncated}, we have
\begin{equation}\label{EQN:JohnsonLpRate}
\dist(f,S_p^h(\phi_h)|_\Omega)_{\LL_p(\Omega)} 
     \lesssim \|f\|_{B_1^\alpha(\LL_p(\Omega))}h^\alpha,
\end{equation}
which completes this part of the proof.

We now establish the first big-$O$ statement using \Cref{thm:positivity-no-clipping}. The approximant to $\chi E(f)$ is constructed as
\[
s_h = \sum_{k=0}^n g_k(2^{n-k}\cdot) *_{h2^{n-k}} f_k,
\]
where $n = n(h)$ is the largest integer with $2^n h \leq 1$, the $f_k$ come from the Littlewood--Paley decomposition of $\chi E(f)$ (defined via a dyadic partition of unity in the Fourier domain), and each $g_k \in S_1(\phi_{r_k})$ where $r_k$ is chosen so that $\|\varphi - g_k(2^{n-k}\cdot)\|_{\LL_\infty(\R^d)} \leq C2^{-\alpha(n-k)}$. Here $\varphi$ is a fixed Schwartz function and $*_{h2^{n-k}}$ denotes semi-discrete convolution on the lattice $h2^{n-k}\Z^d$.  

For $\chi E(f)\in B_1^\alpha(\LL_p(\R^d))$, $\{f_k\}\subset\ell_p$, and since $g_k\in S_1(\phi_{r_k})$, $s_h\in S_p(g_k)\subset S_p(\phi_h)$ since $\ell_1\ast \ell_p\subset\ell_p$ by Young's inequality.

Finally, for the Gaussian generator, condition~(2.4) of~\cite{johnson1997approximation} is verified with $p=\infty$ (Example~4.1), so the error bound~(6.3) therein yields, for the \emph{same} approximant $s_h\in S_p^h(\phi_h)$,
\begin{equation}\label{EQN:JohnsonLinftyRate}
\|f-s_h|_\Omega\|_{\LL_\infty(\Omega)} \leq \|\chi E(f)-s_h\|_{\LL_\infty(\R^d)} \lesssim \|f\|_{B_1^\alpha(\LL_\infty(\Omega))}h^\alpha.
\end{equation}
In particular, for $h\leq h_0(f,m,\Omega,\alpha)$ sufficiently small, $\|f-s_h|_\Omega\|_{\LL_\infty(\Omega)}\leq m/2$. Setting $A(f):= s_h|_\Omega$, $Z:=\int_\Omega A(f)$, and $\widetilde{A}(f):= Z^{-1}A(f)$, an application of \Cref{thm:positivity-no-clipping} yields
\[\dist(f\dd x, S_p^h(\phi_h)|_\Omega \dd x)_{\W_p(\Omega)} \lesssim \|f\|_{\LL_p(\Omega)}\dist(f,S_p^h(\phi_h)|_\Omega)_{\LL_p(\Omega)},\]
which completes the proof.
\end{proof}

\begin{remark}\label{RMK:JohnsonLpBesov}
The assumption $f\in B_1^\alpha(\LL_\infty(\Omega))$ in \Cref{COR:JohnsonExample} is used solely to obtain the $\LL_\infty$ bound \eqref{EQN:JohnsonLinftyRate} so that we can use \Cref{thm:positivity-no-clipping} to avoid modifying the approximation space. If one only assumes $f\in B_1^\alpha(\LL_p(\Omega))$, the $\LL_p$ rate \eqref{EQN:JohnsonLpRate} still holds, but there is no guarantee that the approximant from $S_p^h(\phi_h)|_\Omega$ is bounded away from $0$. In this case, one may apply \Cref{THM:clipped-positivity} instead, at the cost of two modifications: first, the approximant is clipped from below at the value $m$, so the resulting density $\widetilde{A}(f)$ need no longer lie in $S_p^h(\phi_h)|_\Omega$, and the approximation is instead from a modified space; second, the implied upper bound will contain a factor of $\|f\|_{\LL_p(\Omega)}^{1+\frac{1}{p'}}$ rather than $\|f\|_{\LL_p(\Omega)}$ owing to the use of \Cref{THM:clipped-positivity}.
\end{remark}

\subsubsection{Cardinal Interpolation with Gaussians}\label{SEC:GaussianCardinal}

Taking a somewhat different convention from Johnson's paper, let $\phi_h(x) = e^{-h|x|^2}$ for $h>0$ and $x\in\R^d$. Hangelbroek, Madych, Narcowich, and Ward \cite[Theorem 2.1]{hangelbroek2012cardinal} prove the following.
\begin{theorem}[\!\!{\cite[Theorem 2.1]{hangelbroek2012cardinal}}]\label{THM:HMNW}
For every $f\in \Sob{k}{p}(\R^d)$, $h\in(0,1]$ there is a \textit{cardinal Gaussian interpolant} $I_hf\in S_p^h(\phi_{h^2})$ such that $I_hf|_{h\Z^d} = f|_{h\Z^d}$ and
\[\|I_hf-f\|_{\Sob{k}{p}(\R^d)} \lesssim \begin{cases} \|f\|_{\Sob{k}{p}(\R^d)}h^k, & p\in(1,\infty), k>d/p\\
\|f\|_{\Sob{k}{p}(\R^d)}(1+|\log h|)^dh^k, & p= 1, k\geq d, \textnormal{ or } p=\infty, k>0.\end{cases} \]
\end{theorem}

Note that if $\Omega\subset\R^d$ is a bounded Lipschitz domain, then there exists a bounded extension operator $E:\Sob{k}{p}(\Omega)\to\Sob{k}{p}(\R^d)$ \cite{jones1981quasiconformal}. So we may use an argument similar to that above for optimal approximation rates for nonstationary Gaussian shift-invariant ladders to prove a result on cardinal interpolation by Gaussians as follows.

\begin{corollary}\label{COR:CardinalGaussian}
   Take the joint notations and assumptions of \Cref{THM:ApproximationTransfer,thm:positivity-no-clipping}, and let $p, k, d$ be as in \Cref{THM:HMNW}. If $\Omega$ has Lipschitz boundary, then for every $f\in \Sob{k}{\infty}(\Omega)$ with $\int_\Omega f=1$ and $f\geq m>0$ a.e., there exists a density $I_h f\in S_p^h(\phi_{h^2})|_\Omega$ such that $I_hf|_{h\Z^d}=f_{h\Z^d}$ and, for sufficiently small $h>0$,
   \[\W_p(f\dd x, I_hf\dd x) \lesssim \begin{cases}
       \|f\|_{\Sob{k}{p}(\R^d)}h^k, & p\in(1,\infty), \ k>d/p\\
\|f\|_{\Sob{k}{p}(\R^d)}(1+|\log h|)^dh^k, & p= 1, \ k\geq d.\end{cases} \]
\end{corollary}
\begin{proof}
    The proof follows as that of \Cref{COR:JohnsonExample} \textit{mutatis mutandis} by applying the extension operator from $\Sob{k}{p}(\Omega)\to\Sob{k}{p}(\R^d)$, the fact that $\Sob{k}{\infty}(\Omega)\subset\Sob{k}{p}(\Omega)$ since $\Omega$ is bounded, and \Cref{THM:HMNW} to \Cref{thm:positivity-no-clipping}. Note that the extra integrability assumption that $f\in \Sob{k}{\infty}(\Omega)$ is for similar reasons that one needs $\|I_hf-f\|_{\LL_\infty(\Omega)}\leq m/2$ to apply \Cref{thm:positivity-no-clipping}, which requires the use of \Cref{THM:HMNW} for $p=\infty$ (not requiring any adjustment for $k$).
\end{proof}

\begin{remark}
    Similar results for cardinal interpolation with general inverse multiquadrics in $S_p^h(\phi_{\beta,h})$ with $\phi_{\beta,h}(x) := (|x|^2+h^{-2})^\beta$, for fixed $\beta < -d-\frac12$, can be obtained similarly from \cite[Theorem 3.5]{hamm2018cardinal}. The only difference in the rates compared to \Cref{COR:CardinalGaussian} is that there is no power on the $(1+|\log h|)$ term for the case $p=1$.
\end{remark}

It should be noted that the procedure illustrated above by these examples is quite general, and therefore many known results in RBF approximation can be transferred in this manner to approximation of densities; e.g., surface splines as in \cite{johnson20012}, and exponential box splines \cite{johnson1997approximation}, among others.

\subsection{Application to kernel density estimation}\label{SEC:KDEExample}

As a final application of \Cref{THM:ApproximationTransfer} we show that our approximation theory can give Wasserstein distance bounds for random approximation schemes by kernel density estimators (KDEs). In particular, we show that our transfer principle (\Cref{THM:ApproximationTransfer}) converts any
$\LL_p$ risk bound for a density estimator into a $\W_p$
convergence rate (either in expectation or probability). The main point of the following proposition is that the estimator, the smoothness class, and the rate
enter only through a single risk bound; nothing else about the specific construction is
utilized.

In what follows, for an estimator $\tilde f$ built from $n$ i.i.d.~samples of $f$, we write the risk as
\[
    R_n^{(p,q)}(\tilde f, f) := \left(\E\,\|\tilde f - f\|_{\LL_p(\Omega)}^q\right)^{1/q},
\]
where the expectation is taken over the i.i.d.~samples of $f$.

\begin{proposition}\label{PROP:StochasticTransfer}
Let $\Omega\subset\R^d$ be a bounded Lipschitz domain, $p\in(1,\infty)$, and $q\geq 1$. Let $f$ be a probability
density on $\Omega$ with $f\geq m>0$ a.e. Let $(\widehat f_n)_{n\in\N}$ be a sequence
of random measurable functions on $\Omega$ and $(\psi_n)_{n\in\N}\subset\R_+$ be such that
\[
    R_n^{(p,q)}(\widehat f_n, f) \leq \psi_n .
\]
Define
\[
    u_n := \max\Big\{\widehat f_n,\tfrac m2\Big\},\qquad
    Z_n := \int_\Omega u_n,\qquad
    \widetilde f_n := Z_n^{-1}u_n .
\]
Then there exist constants $C,c_0>0$, depending on $\Omega,d,p$ and $m$, such that for every
$\delta\in(0,1)$ and every $n$ with $\psi_n\,\delta^{-\frac1q}\leq c_0$,
\[
    \W_p\big(f\dd x,\widetilde f_n\dd x\big)\leq C\psi_n\,\delta^{-\frac1q}
    \qquad\text{with probability at least } 1-\delta.
\]
\end{proposition}

\begin{proof}
Write $E_n:=\|\widehat f_n-f\|_{\LL_p(\Omega)}$, $r_{n,\delta}:=\psi_n\,\delta^{-1/q}$, and note that Markov's inequality implies that
\[
    \Prob(E_n> r_{n,\delta})
    = \Prob\left(E_n^q> \psi_n^q\,\delta^{-1}\right)
    \leq \frac{\E[E_n^q]}{\psi_n^q\,\delta^{-1}}
    \leq \delta,
\]
so $\Prob(\|\widehat f_n-f\|_{\LL_p(\Omega)}\leq r_{n,\delta})\geq1-\delta$. The remainder of the proof holds on the event $\Omega_n:=\{\|\widehat f_n-f\|_{\LL_p(\Omega)}\leq r_{n,\delta}\}.$ 

Now mimic the proof of \Cref{THM:clipped-positivity} and notice that, pointwise, $|u_n-f|\leq|\widehat f_n-f|$. Hence, $\|u_n-f\|_{\LL_p(\Omega)}\leq E_n$, and by H\"{o}lder's inequality,
\[
    |Z_n-1|=\Big|\int_\Omega(u_n-f)\Big|\leq |\Omega|^{\frac{1}{p'}}\|u_n-f\|_{\LL_p(\Omega)}\leq |\Omega|^{\frac{1}{p'}}E_n.
\]
Choosing $c_0:=\min\{1,\frac12|\Omega|^{-\frac{1}{p'}}\}$, the condition $r_{n,\delta}\leq c_0$
gives $E_n\leq 1$ and $Z_n\leq \tfrac32$, whence
\[
    \widetilde f_n = Z_n^{-1}u_n \geq \frac m3>0\quad\text{a.e.}
\]
As in the proof of \Cref{THM:clipped-positivity}, from this we obtain the bound
$\|f-\widetilde f_n\|_{\LL_p(\Omega)}\lesssim E_n$, whereupon application of \Cref{THM:ApproximationTransfer} with $\mc{F}=\{f\}$ and $\mc{A} = \{\widetilde{f}_n\}$ yields
\[
    \W_p\left(f\dd x,\widetilde f_n\dd x\right)
    \lesssim \|f-\widetilde f_n\|_{\LL_p(\Omega)}
    \lesssim E_n \lesssim r_{n,\delta},
\]
which is the desired conclusion.
\end{proof}

Note that the construction above gives a probabilistic bound, but the same method of proof yields the expectation bound
\[\E\W_p(f\dd x,\widetilde{f}_n\dd x)\lesssim \E\|f-\widehat f_n\|_{\LL_p(\Omega)}=R_n^{(p,1)}(\widehat f_n,f)\]
under the same assumptions.

Much of the KDE literature has focused on estimating densities on $\R^d$ or on an interval $[a,b]\subset\R$. When estimating a density on a bounded subdomain of $\R^d$, standard KDEs often suffer from so-called \textit{boundary bias}, in which the estimator may exhibit slower decay near the boundary for a density that does not decay at the boundary. Since our framework requires our densities to be bounded below, we must consider estimators that mitigate this boundary bias. Some approaches to this are through reflection \cite{schuster1985incorporating,silverman1986kernel}, and asymmetric kernels whose support adapts to the domain. The latter is natural in our setting given that the densities do not vanish at the boundary, and asymmetric kernels are constructed to mitigate the boundary bias without extending beyond the domain. In dimension one, the beta kernel of Chen \cite{chen1999beta} provides such an estimator on $[0,1]$, while Bouezmarni and Rombouts \cite{bouezmarni2010nonparametric} construct boundary bias free estimators on $Q_d:=[0,1]^d$ by using beta kernels.

We now provide a couple of examples of applications of \Cref{PROP:StochasticTransfer}. In each case the only input is an $\LL_p$ risk bound on a domain where $f$ is bounded away from 0. Our first example is an adaptive boundary KDE of Bertin et al.~\cite{bertin2019adaptive}, who consider a boundary corrected
KDE with adaptive bandwidth selection that approximates a density $f$ on the unit cube $Q_d$ from $n$ i.i.d. samples $x_1,\dots,x_n$. The isotropic boundary kernel density estimator $\widehat{f}^{\text{iso}}$
therein is a data-driven estimator on
$Q_d$ that jointly selects both a kernel order and a bandwidth, building on works by Goldenshluger and Lepski \cite{goldenshluger2014adaptive}. We do not reproduce its construction here, but refer the interested reader to \cite[Section 3]{bertin2019adaptive}
for details. The key result we use is their Theorem 4, which says that for any $s>0$, $L > 0$, and $p, q \geq 1$,
\begin{equation}\label{EQN:BertinTheorem4}
    \limsup_{n\to\infty} n^{\frac{s}{2s+d}}
    \sup_{f \in \widetilde S_{s,p}(L)} R_n^{(p,q)}\big(\widehat{f}^{\text{iso}}, f\big) < \infty.
\end{equation}
Here, $\widetilde S_{s,p}(L)$ denotes the isotropic Sobolev--Slobodetskii
ball (roughly, functions whose derivatives of order $\leq \lfloor s\rfloor$ are in $\LL_p(Q_d)$ and whose Gagliardo semi-norm with parameters $s,p$ is bounded by $L$ \cite[Definition 2]{bertin2019adaptive}). The authors also note that \eqref{EQN:BertinTheorem4} holds for densities in an anisotropic H\"{o}lder space when $s>\frac1p$. 

\begin{corollary}\label{COR:KDE}
Let $p\in(1,\infty)$, $s>0$, $q\geq1$, and $L,m>0$. Let $f\in\widetilde S_{s,p}(L)$
on $Q_d$ satisfy $\int_{Q_d}f=1$ and $f\geq m$ a.e., let $x_1,\dots,x_n$ be
i.i.d.\ from $f\dd x$, and let $\widehat f^{\mathrm{iso}}$ be the isotropic boundary
kernel density estimator of \cite[Section~3.2]{bertin2019adaptive}. Form
$\widetilde f_n$ as in \Cref{PROP:StochasticTransfer}. Then there exist $C>0$ and
$N_0\in\N$, depending on $d,p,q,s,L,m$, such that for every $\delta\in(0,1)$ and
every $n\geq N_0$ with $C_0\,\delta^{-1/q}n^{-s/(2s+d)}\leq c_0$,
\[
    \W_p\big(f\dd x,\widetilde f_n\dd x\big)\leq C\,\delta^{-1/q}\,n^{-\frac{s}{2s+d}}
    \qquad\text{with probability at least } 1-\delta.
\]
\end{corollary}

\begin{proof}
By \cite[Theorem~4]{bertin2019adaptive}, \eqref{EQN:BertinTheorem4} holds, so there
exist $C_0,N_0$ with $R_n^{(p,q)}(\widehat f^{\mathrm{iso}},f)\leq C_0n^{-\frac{s}{2s+d}}$
for $n\geq N_0$. Apply \Cref{PROP:StochasticTransfer} with $\psi_n=C_0n^{-\frac{s}{2s+d}}$.
\end{proof}

\begin{remark}\label{REM:KDEDiscussion}
By \cite[Theorem~4 lower bound]{bertin2019adaptive}, the rate
$n^{-\frac{s}{2s+d}}$ is minimax optimal for $\LL_p(Q_d)$ density estimation
over $\widetilde S_{s,p}(L)$. The transfer of this rate to $\W_p$ via
\Cref{THM:ApproximationTransfer} does not, however, recover the
$\W_p$-minimax rate $n^{-\frac{1+s}{2s+d}}$ established in
\cite[Theorem~3]{niles2022minimax}. The gap reflects the fact that our technique of transferring $\LL_p$ bounds to $\W_p$ does not optimally exploit smoothness information in the densities. \Cref{COR:KDE} is nonetheless interesting in that it shows our framework can be applied to data-driven estimators.
\end{remark}

Our second eample is a bound in expectation for a multivariate beta kernel from the work of Bouezmarni and Rombouts \cite{bouezmarni2010nonparametric}. One can turn the following into a probabilistic bound in the same way as above.

\begin{corollary}\label{COR:KDEBeta}
Let $m>0$ and let $f\in\CC^2(Q_d)$ satisfy $\int_{Q_d}f=1$ and $f\geq m$
a.e. Let
$\widehat f_n$ be the multivariate beta-kernel density estimator of
\cite{bouezmarni2010nonparametric} with bandwidth $b\asymp n^{-\frac{2}{4+d}}$, and form
$\widetilde f_n$ as in \Cref{PROP:StochasticTransfer}. Then there is a constant
$C=C_{d,m}$ such that for all sufficiently large $n$,
\[
    \E\,\W_2\big(f\dd x,\widetilde f_n\dd x\big)\leq C\,n^{-\frac{2}{4+d}} .
\]
\end{corollary}

\begin{proof}
By \cite{bouezmarni2010nonparametric}, the beta-kernel estimator is free of
boundary bias and its mean integrated squared error satisfies
$\E\|\widehat f_n-f\|_{\LL_2(Q_d)}^2 = O\left(n^{-\frac{4}{4+d}}\right)$. By Jensen's
inequality, $R_n^{(2,1)}(\widehat f_n,f)\leq\big(\E\|\widehat f_n-f\|_{\LL_2(Q_d)}^2\big)^{\frac12}
= O\left(n^{-\frac{2}{4+d}}\right)$. Apply the expectation form of
\Cref{PROP:StochasticTransfer} with $p=2$ and $\psi_n\asymp n^{-\frac{2}{4+d}}$.
\end{proof}

\subsection{Piecewise approximation of densities}

As discussed previously, sometimes it is desirable to break up approximation of a density on a domain $\Omega$ into smaller pieces (e.g., if $\Omega$ is large). Thus, we want to understand how to locally approximate the density $f$ by some approximation scheme (e.g., piecewise polynomial approximation rather than global polynomial approximation). We prove below a general theorem in this vein, but do not give further specific examples, as it is a general framework, and any examples from the previous list would work here.

\begin{theorem}
\label{thm:piecewise-approximation}
Let $p \in (1,\infty)$, and $\Omega \subset \R^d$ be a bounded Lipschitz domain. Let $f \in \LL_p(\Omega)$
satisfy $\int_\Omega f = 1$ and $f(x) \geq m > 0$ for a.e. $x \in \Omega$. Let
$X = \{x_i\}_{i=1}^N \subset \Omega$ be quasi-uniform with associated Voronoi partition
$\{V_i\}_{i=1}^N$ of $\Omega$.

For each $i$, let $A_i(f) \in \LL_p(V_i)$ be a local approximation of $f|_{V_i}$
satisfying
\[
  \|A_i(f) - f\|_{\LL_\infty(V_i)} \leq \frac{m}{2}.
\]
Define $g_0 \in \LL_p(\Omega)$ by $g_0(x) := A_i(f)(x)$ for $x \in V_i$, set
$Z := \int_\Omega g_0 \dd x$, and $g := Z^{-1} g_0$. Then $\int_\Omega g \dd x = 1$,
$g \geq m/(2Z) > 0$ a.e., and
\[
  \W_p(f \dd x, g \dd x)
  \leq C_{\Omega,d,p,m} \|f\|_{\LL_p(\Omega)}\left(\sum_{i=1}^N \|f - A_i(f)\|^p_{\LL_p(V_i)}\right)^{\frac1p}.
\]
If $p=1$, then $\Omega$ need not be Lipschitz, and $\|f\|_{\LL_1(\Omega)}=1$ in the final bound.
\end{theorem}

\begin{proof}
First note that for a.e. $x\in\ V_i$,
\[
  g_0(x) = A_i(f)(x)
  \geq f(x) - \|A_i(f) - f\|_{\LL_\infty(V_i)}
  \geq m - \frac{m}{2} = \frac{m}{2},
\]
hence $g_0 \geq m/2 > 0$ a.e.\ on $\Omega$. In particular,
$Z = \int_\Omega g_0 \dd x \geq \frac{m}{2}|\Omega| > 0$, so $g = Z^{-1}g_0$ is
well-defined, satisfies $\int_\Omega g \dd x = 1$, and $g(x) \geq m/(2Z) > 0$ a.e.

Next, since $\bigsqcup_i V_i = \Omega$,
\begin{equation}\label{eq:g0bound}
  \|f - g_0\|^p_{\LL_p(\Omega)}
  = \sum_{i=1}^N \int_{V_i} |f(x) - A_i(f)(x)|^p \dd x
  = \sum_{i=1}^N \|f - A_i(f)\|^p_{\LL_p(V_i)}
  =: E^p.
\end{equation}
Thus by H\"{o}lder's inequality,
\[
  |Z - 1|
  = \left|\int_\Omega (g_0 - f)\dd x\right|
  \leq \|g_0 - f\|_{\LL_1(\Omega)}
  \leq |\Omega|^{1 - \frac1p} E,
\]
so
\begin{equation}
  \label{eq:Z-bound}
  |1 - Z^{-1}| = \frac{|Z-1|}{Z} \leq \frac{2|\Omega|^{-\frac1p}}{m} E.
\end{equation}
Moreover, by the triangle inequality and \eqref{eq:g0bound},
\begin{equation}
  \label{eq:g0-norm}
  \|g_0\|_{\LL_p(\Omega)} \leq \|f\|_{\LL_p(\Omega)} + E.
\end{equation}

Now we have $f - g = (f - g_0) + (1 - Z^{-1})g_0$, whereby \eqref{eq:g0bound}--\eqref{eq:g0-norm} imply
\[
  \|f - g\|_{\LL_p(\Omega)}
  \leq E + \frac{2|\Omega|^{-1/p}}{m} E \left(\|f\|_{\LL_p(\Omega)} + E\right).
\]
Since $f \geq m$ on $\Omega$ we have $\|f\|_{\LL_p(\Omega)} \geq m|\Omega|^{1/p}$,
so $1 \leq (m|\Omega|^{1/p})^{-1}\|f\|_{\LL_p(\Omega)}$. Substituting and
collecting terms yields
\begin{equation}
  \label{eq:lp-after-correction}
  \|f - g\|_{\LL_p(\Omega)}
  \leq C_{\Omega,p,m}\|f\|_{\LL_p(\Omega)} E
\end{equation}
where we also used the fact that $E\leq C_{\Omega,p,m}$ by the $\LL_\infty$ bound in the assumptions..

Finally, we have $f,g\in \LL_p(\Omega)$ with unit integral and $f,g \geq \widetilde{m}>0$ where $\widetilde{m}:=\min\{m,\frac{m}{2Z}\}$ depends only on $\Omega$, $p$, and $m$.  Therefore, applying \Cref{thm:positivity-no-clipping} gives
\[
  \W_p(f \dd x, g \dd x) \leq C_{\Omega,d,p,\widetilde{m}} \|f - g\|_{\LL_p(\Omega)} \leq C_{\Omega,d,p,m}\|f\|_{\LL_p(\Omega)}E,
\]
which completes the proof for $p\neq1.$ For $p=1$ the proof follows from simple modifications of the bounds above but applying \Cref{THM:TV} to conclude the upper bound on $\W_1$.
\end{proof}

Again, the local $\LL_\infty$ condition is used solely to guarantee positivity of $g_0$
without modifying the approximation space. If only $\LL_p$ local bounds are
available, one may instead define $g_0(x) := \max\{A_i(f)(x), m\}$ on each $V_i$
and apply \Cref{THM:clipped-positivity} in place of \Cref{thm:positivity-no-clipping} in the final step, at the cost of replacing
$\|f\|_{\LL_p}$ by $\|f\|_{\LL_p}^{1+\frac{1}{p'}}$ in the bound.

\begin{remark}\label{REM:EUpperBound}
From the error term $E$ in \Cref{thm:piecewise-approximation}, we can see a more nuanced view of the trade-off of the local approximation rate on each Voronoi cell with the volume of the Voronoi cell. Recall that $A_i(f)$ approximates $f|_{V_i}$, and suppose that for each $i$
\begin{align*}
    \| A_i(f) - f \|_{\LL_\infty(V_i)} \leq \frac{m_i}{2},
\end{align*}
with $\max_i m_i\leq m$. Then the error term can be bounded as
\[
    E^p = \sum_{i=1}^N \|f - A_i(f) \|_{\LL_p(V_i)}^p  \leq \sum_{i=1}^N \| f - A_i(f) \|_{\LL_\infty(V_i)}^p |V_i| \leq \sum_{i=1}^N \left( \frac{m_i}{2} \right)^p|V_i|.
\]
Now if $r_i = \rad(V_i)$, then we have $|V_i|\leq |B(x_i,r_i)| = C_dr_i^d$ (with $C_d = \pi^{d/2}/\Gamma(d/2+1)$). Thus
\begin{equation}\label{EQN:EBound} E \leq \frac{C_d^\frac1p}{2} \left(\sum_{i=1}^Nm_i^pr_i^d\right)^\frac1p.\end{equation}
Note that $r_i\leq h_X$, so this bound can be simplified to 
\[E\leq \frac12 C_d^\frac1p mh_X^\frac{d}{p}N^\frac1p.\]
However, \eqref{EQN:EBound} illustrates that there is a computational tradeoff between the $\LL_\infty$ approximation with the Lebesgue measure of $V_i$. In particular, it suggests that in an adaptive scheme, one can 
balance computational cost against approximation accuracy by choosing finer 
meshes and/or higher-order approximators in regions where $f$ is difficult to 
approximate, and coarser meshes and/or cheaper approximators where $f$ is well 
behaved, with the bound providing a concrete budget for how 
these local choices contribute to the global error $E$.

\end{remark}

\section{Structured Discrete Approximations}\label{SEC:Voronoi}

In this section, we discuss discrete, $N$-term approximations to measures subordinate to a fixed set of quasi-uniform centers. We focus on the most general case of nonuniform approximation on $\R^d$ first, and specialize it to approximation at full-rank lattices, which is similar to the notion of lattice quantizers in \cite[Section 8.3]{graf2007foundations}, and then to finite point sets. To wit, let $X:=\{x_i\}_{i=1}^\infty\subset\R^d$ be a set of points that is quasi-uniform: $0<q_X\leq h_X<\infty$, where $q_X$ and $h_X$ are defined by \eqref{EQN:SeparationRadius} and \eqref{EQ:MeshNorm}, respectively. Let $\{V_i\}_{i=1}^\infty$ be the Voronoi cells associated with $X$.

We consider a general type of approximation as follows:
\begin{equation}\label{EQN:CameraGeneral}\mu_X:=\sum_{i=1}^\infty \alpha_i\tau_i\end{equation}
where each $\tau_i$ is a measure supported on $V_i$ and $\alpha_i\geq0$. An obvious choice to make $\mu_X$ a probability measure is to take $\alpha_i = \mu(V_i)/\tau_i(V_i)$, where we take the convention that $\mu(V_i)/\tau_i(V_i) = 0$ if the denominator is 0. 

Two special cases of this approximation scheme are motivated by orthographic projection camera models and correspond to piecewise constant approximation:
\[\sum_{i=1}^\infty\beta_i\delta_{x_i},\quad \textnormal{and}\quad \sum_{i=1}^\infty\gamma_i\one_{V_i}.\]

Cameras naturally map light intensities to a uniform grid in $\R^2$ (though we present our results in $\R^d$ for generality). In an orthographic projection camera model, one assumes the existence of a camera plane onto which the scene is projected orthogonally, in contrast to a perspective camera. Such a model is approximately accurate for objects at a medium distance from the camera \cite{stockman2001computer}. The models above form a measure from a camera image by either assigning the pixel intensity value $\beta_i$ to a Dirac spike at the center of the grid point on the camera plane, or by spreading it out over the whole pixel (Voronoi cell) as $\gamma_i\one_{V_i}(\cdot)$. If the scene being imaged is normalized and represented as a density, then both of these approximations are analogous to piecewise constant approximations of the density, although the measures are constructed differently. We will see later that both of these schemes, and the more general model \eqref{EQN:CameraGeneral} exhibit the same rate of approximations as the empirical measure for compactly supported measures under some assumptions on the separation radius and mesh norm of $X$. Throughout this section, we will assume that the Voronoi cells associated with $X$ are pairwise disjoint as mentioned in \Cref{SEC:Background}.

\subsection{Main Technical Lemma}

Our main argument for showing the approximation power of the schemes introduced above involves forming a nonoptimal coupling (in the Kantorovich sense) between $\mu$ and $\mu_X$ and estimating the cost of the coupling in the Kantorovich formulation of $\W_p$. This requires the estimation in the following lemma. Here, we use $\|x\|_{\LL_\infty(V_{i})}$ to be the $\LL_\infty$ norm of the identity function $f(x)=x$ on a Voronoi cell, which in this case is $\max_{x\in V_{i}}|x|$. 

\begin{lemma}\label{LEM:RadialEstimatesNonuniform}
Let $\mu\in \W_p(\R^d)$, $p\in[1,\infty)$. Let $X\subset\R^d$ be such that $h_X<\infty$, with Voronoi cells $\{V_i\}_{i=1}^\infty$. Then the following hold:
\begin{enumerate}[(i)]
\item\label{ITEM:NonuniformSumIntegral} $\displaystyle \sum_{i=1}^\infty|x_i|^p\mu(V_{i})\leq 2^{p-1}h_X^p+2^{p-1}M_p(\mu),$

\bigskip
\item\label{ITEM:NonuniformSumInfty} $\displaystyle\sum_{i=1}^\infty\|x\|_{\LL_\infty(V_{i})}^p\mu(V_{i})\leq 2^{p-1}\sum_{i=1}^\infty|x_i|^p\mu(V_{i}) + 2^{p-1}h_X^{p}$,

\bigskip
\item\label{ITEM:NonuniformCombined} $\displaystyle\sum_{i=1}^\infty\|x\|_{\LL_\infty(V_{i})}^p\mu(V_{i})\leq (2^{2p-2}+2^{p-1})h_X^p + 2^{2p-2}M_p(\mu)$.
\end{enumerate}
\end{lemma}

\begin{proof}
Proof of \eqref{ITEM:NonuniformSumInfty}: Note that
\begin{align*}
\sum_{i=1}^\infty \|x\|_{\LL_\infty(V_i)}^p\mu(V_i) & \leq \sum_{i=1}^\infty (|x_i|+h_X)^p\mu(V_i)\\
& \leq 2^{p-1}\sum_{i=1}^\infty|x_i|^p\mu(V_i)+2^{p-1}h_X^p\sum_{i=1}^\infty\mu(V_i)\\
& = 2^{p-1}\sum_{i=1}^\infty|x_i|^p\mu(V_i)+2^{p-1}h_X^p.
\end{align*}

Proof of \eqref{ITEM:NonuniformSumIntegral}: By the definition of the mesh norm and the triangle equality, the following holds for every $i$ and every $x\in V_i$: \[|x_i|\leq |x|+|x-x_i|\leq |x|+h_X,\]
hence $|x_i|^p\leq 2^{p-1}(|x|^p+h_X^p).$
Integrating this inequality over $V_i$ with respect to $\mu$ and summing over $i$ gives
\[\sum_{i=1}^\infty|x_i|^p\mu(V_i) \leq 2^{p-1}\sum_{i=1}^\infty\int_{V_i}|x|^p\dd\mu(x)+2^{p-1}h_X^p\sum_{i=1}^\infty\mu(V_i) = 2^{p-1}M_p(\mu)+2^{p-1}h_X^p,\]
which is the desired conclusion.

Proof of \eqref{ITEM:NonuniformCombined}: Combine \eqref{ITEM:NonuniformSumIntegral} and \eqref{ITEM:NonuniformSumInfty}.
\end{proof}

\begin{remark}\label{REM:Latticeh_X} Note that if $X$ is a full-rank lattice on $\R^d$, then each Voronoi cell is a translate of the one centered at $0$, and $h_X = \rad(V_0):= \max_{x\in V_0}|x|$ in the bounds above and hereafter.\end{remark}

\subsection{Approximation Rates for arbitrary measures}\label{SEC:Nonuniform}

Our first main result is on approximation at the full point set $X$.

\begin{theorem}\label{THM:Nonuniform}
Let $\mu\in \W_p(\R^d)$, $p\in[1,\infty)$, and $X\subset\R^d$ be such that $h_X<\infty$ with Voronoi cells $\{V_i\}$. Let $\{\tau_i\}_{i=1}^\infty$ be measures supported on $V_i$, with $\tau_i(V_i)=0$ if and only if $\mu(V_i)=0$. Let $\mu_X:=\sum_{i=1}^\infty \frac{\mu(V_i)}{\tau_i(V_i)}\tau_{i}$ (with $\mu(V_i)/\tau_i(V_i) = 0$ if $\tau_i(V_i) = 0$). Then
\[\W_p(\mu,\mu_X) \leq 2h_X.\]
\end{theorem}
\begin{proof}
First, note that $\mu_X$ is clearly a measure, and we have (since $\supp(\tau_i)\subset V_i$)
\[\mu_X(\R^d) = \sum_{i=1}^\infty \frac{\mu(V_i)}{\tau_i(V_i)}\tau_i(V_i) = \sum_{i=1}^\infty \mu(V_i) = \mu\left(\bigcup_{i=1}^\infty V_i\right) = \mu(\R^d) = 1,\] where the second and third equalities comes from countable additivity of $\mu$ and the fact that $\R^d = \bigsqcup_{i=1}^\infty V_i$. To show that $\mu_X$ has finite $p$-th moment, we notice that (via Tonelli's Theorem) and \Cref{LEM:RadialEstimatesNonuniform}\eqref{ITEM:NonuniformCombined}),
\[\int_{\R^d}|x|^p\dd\mu_X \leq \sum_{i=1}^\infty \frac{\mu(V_i)}{\tau_i(V_i)}\|x\|_{\LL_\infty(V_i)}^p\tau_i(V_i)   \leq (2^{2p-2}+2^{p-1})h_X^p+2^{2p-2}M_p(\mu)<\infty.\]

Using the Kantorovich formulation of $\W_p$, we define the following (non-optimal) coupling of $\mu$ and $\mu_X$:
\[\widetilde{\pi}(A,B):= \sum_{i=1}^\infty \frac{\mu(A\cap V_i)}{\tau_i(V_i)}\tau_i(B\cap V_i) = \int_{A\times B}\sum_{i=1}^\infty \frac{1}{\tau_i(V_i)} \one_{V_i}(x)\one_{V_i}(y)\dd\mu(x)\dd\tau_i(y).\]
It is straightforward to check that $\widetilde{\pi}$ is a measure on $\R^d\times\R^d$. Noting that
\[\int_{\R^d\times \R^d}\dd\widetilde{\pi}(x,y) = \sum_{i=1}^\infty\mu(V_i) = \mu(\R^d)=1,\]
we see that $\widetilde{\pi}$ is a probability measure, and its marginals are
\[\widetilde{\pi}(A,\R^d) = \sum_{i=1}^\infty \frac{\mu(A\cap V_i)}{\tau_i(V_i)}\tau_i(V_i) = \sum_{i=1}^\infty \mu(A\cap V_i) = \mu(A),\]
and
\[\widetilde{\pi}(\R^d,B) = \sum_{i=1}^\infty \frac{\mu(V_i)}{\tau_i(V_i)}\tau_i(B\cap V_i) = \mu_X(B),\] for all Borel measurable sets $A,B\in\R^d$. Therefore, $\widetilde{\pi}$ is a coupling of $\mu$ and $\mu_X$.

Notice that if $k\neq j$, then
\[\widetilde{\pi}(V_k,V_j) = \sum_{i=1}^\infty \frac{\mu(V_k\cap V_i)}{\tau_i(V_i)}\tau_i(V_j) = \frac{\mu(V_k)}{\tau_k(V_k)}\tau_k(V_j) = 0,\]
that is, $\widetilde{\pi}$ only evaluates mass on sets intersecting $V_i\times V_i$. Therefore, we have
\begin{align*}
    \W_p(\mu,\mu_X)^p & \leq \int_{\R^d\times \R^d}|x-y|^p\dd\widetilde{\pi}(x,y)\\
    & = \sum_{i=1}^\infty \int_{V_i\times V_i}|x-y|^p\dd\widetilde{\pi}(x,y)\\
    & = \sum_{i=1}^\infty \frac{1}{\tau_i(V_i)}\int_{V_i\times V_i}|x-y|^p\dd\mu(x)\dd\tau_i(y)\\
    & \leq 2^ph_X^p\sum_{i=1}^\infty \mu(V_i)\\
    & = 2^ph_X^p,
\end{align*}
which is the desired conclusion.
\end{proof}

\begin{remark} As in \Cref{REM:Latticeh_X}, the upper bound reads $\W_p(\mu,\mu_X) \leq \diam(V_0)$ when $X$ is a full-rank lattice in $\R^d$.
\end{remark}

\subsection{$N$-term approximation rates}

Now we turn to the problem of determining $N$-term approximation rates of compactly supported measures by discrete measures, or more general quantized measures. When we say that an approximating measure $\mu_N$ is an at most $N$-term approximation, we mean that, if $\mu_N$ is discrete, then there are no more than $N$ atoms, or if $\mu_N$ takes the form $\sum_i \frac{\mu_i(V_i)}{\tau_i(V_i)}\tau_i$ where $\tau_i$ are not atomic, then there are no more than $N$ Voronoi cells comprising the support of $\mu_N$. Thus our approximations are either $N$-term discrete measures or are piecewise approximations on $N$ Voronoi cells. Here, we show that with a deterministic approximation scheme, one obtains the same approximation rates as that of optimal quantization. Our first task is to understand how many Voronoi cells are contained in the support of a measure given a set of centers $X$.

Given $X\subset\R^d$ and $R>0$, let \[N_{X,R} := |\{i:V_i\cap \BB_R\neq\emptyset\}|\] be the number of Voronoi cells associated with $X$ that intersect $\BB_R$. Then we have the following counting based on the minimal separation radius and mesh norm of $X$.

\begin{lemma}\label{LEM:PointPacking}
Let $X\subset\R^d$ be quasi-uniform. Then for $R>h_X$,
\[\left(\frac{R-h_X}{h_X}\right)^d \leq N_{X,R} \leq \left(\frac{R+h_X+q_X}{q_X}\right)^d.\]
\end{lemma}
\begin{proof}
For the upper bound, note that if the Voronoi cell $V_i$ intersects $\BB_R$, then there is some $y\in V_i\cap \BB_R$, and we have $|x_i|\leq |y|+|x_i-y|\leq R + h_X.$ Consequently, $N_{X,R}\leq|X\cap\BB_{R+h_X}|.$

Now for every $x_i\in X\cap \BB_{R+h_X}$ we have $\BB(x_i,q_X)\subset \BB_{R+h_X+q_X}$, where the balls around $x_i$ are disjoint by definition of $q_X$. Thus
\[\left|\bigcup_{x_i\in X\cap\BB_{R+h_X}}\BB(x_i,q_X)\right| = \sum_{x_i\in X\cap\BB_{R+h_X}} |\BB(x_i,q_X)| = |X\cap \BB_{R+h_X}||\BB_{q_X}| \leq |\BB_{R+h_X+q_X}|,\]
whence \[N_{X,R} \leq \frac{|\BB_{R+h_X+q_X}|}{|\BB_{q_X}|} = \left(\frac{R+h_X+q_X}{q_X}\right)^d.\]

To see the lower bound, first note that $N_{X,R} \geq |X\cap\BB_R|$ since some Voronoi cells for $x_i$ outside of $\BB_R$ could intersect $\BB_R$. Next, we must have
\[\BB_{R-h_X}\subset \bigcup_{x_i\in X\cap\BB_R} \BB(x_i,h_X),\]
where by similar argument as above, we see that
\[N_{X,R} \geq \frac{|\BB_{R-h_X}|}{|\BB_{h_X}|} = \left(\frac{R-h_X}{h_X}\right)^d.\]
\end{proof}

With this in hand, we have the following result on $N$-term approximation of compactly supported measures.

\begin{corollary}\label{COR:NonuniformCompact}
Take the notations and assumptions of \Cref{THM:Nonuniform}. Suppose also that $\mu$ has compact support contained in $\BB_R$ for $h_X<R<h_X(N^\frac1d-2)$, and that $h_X\leq CN^{-\frac1d}$ for some $0<C<\infty$, and $q_X(N^{\frac1d}-1)\geq R+h_X$. Then $\mu_X = \sum_{i}\frac{\mu(V_i)}{\tau_i(V_i)}\tau_i$ is an at most $N$-term approximation of $\mu$ satisfying
\[\W_p(\mu,\mu_X)\leq 2CN^{-\frac1d}.\]
\end{corollary}

\begin{proof}
    The upper bound follows directly from \Cref{THM:Nonuniform} and the assumption on $h_X$, whereas the bound on $q_X$ implies that $N_{X,R}\leq N$ by \Cref{LEM:PointPacking}.
\end{proof}

\begin{remark}
Note that the bounds of \Cref{LEM:PointPacking} imply that $q_X\gtrsim N^{-\frac1d}$, and since $h_X\geq q_X$ we must have $h_X =\Theta(N^{-\frac1d})$. That is, we cannot improve the order of $N$ in the upper bound on $\W_p(\mu,\mu_X)$ in \Cref{COR:NonuniformCompact}.
\end{remark}

For further illustration of the results above, let us consider the case when $X$ is a lattice or a finite sampling of a lattice. Inspired by the orthographic projection camera model, we consider a regular grid (i.e., integer lattice) in which we can explicitly compute the mesh norm.

\begin{corollary}\label{COR:NpixelCompact}
Take the notations and assumptions of \Cref{THM:Nonuniform}. Suppose that $\mu\in\W_p(\R^d)$, $p\in[1,\infty)$ has support contained in $\left[-\frac12,\frac12\right]^d$, and let $\Lambda_N = \frac{1}{M}(\Z^d+s)$ where $M=\lfloor N^\frac1d\rfloor$, and $s = 0$ if $M$ is odd and $s=\frac12$ if $M$ is even, and let the associated Voronoi cells be $(V_\lambda)_{\lambda\in\Lambda_N}$. Then $\mu_{\Lambda_N} = \sum_{\lambda\in\Lambda_N}\frac{\mu(V_{\lambda})}{\tau_{\lambda}(V_{\lambda})}\tau_{\lambda}$ is an at most $N$-term approximation to $\mu$ that satisfies
\[\W_p(\mu,\mu_{\Lambda_N})\leq \sqrt{d}N^{-\frac1d}.\]
\end{corollary}
\begin{proof}
The choice of the scaling factor and shift on the integer lattice ensures that there are at most $N$ Voronoi cells that intersect the unit cube. Additionally, for any $\alpha\geq0$, we have $h_{\alpha\Z^d} = \alpha h_{\Z^d} = \frac{\alpha\sqrt{d}}{2}$ (and shifts do not change the mesh norm). Applying the conclusion of \Cref{THM:Nonuniform} and simplifying yields the result.
\end{proof}

\begin{remark}
Two particular cases of interest in the general formulation of \eqref{EQN:CameraGeneral} above are $\tau_{i} = \delta_{x_i}$ and $\tau_{i} = \one_{V_{i}}$. Notice that when $\tau_{i} = \delta_{x_i}$, it is easy to see that $\mu_N = \sum_{i} \mu(V_{i}) \delta_{x_i}$.  In the case with $\tau_{i} = \one_{V_{i}}$, we see that $\mu_N = \sum_{i} \frac{\mu(V_{i})}{|V_{i}|} \one_{V_{i}}$.  

Both of these schemes are natural analogues of piecewise constant approximants for measures, and can be related to an orthographic projection camera model. In particular, when $d=2$, the Dirac approximation corresponds to assigning an intensity value to the center of each pixel mapped to a uniform grid on the camera plane (see \cite{hamm2023wassmap} for similar constructions in $\W_2$). \Cref{COR:NpixelCompact} then states that an $N$-pixel camera gives an approximation to $\mu$ in $\W_p$ with error $N^{-\frac12}$. The indicator function approximant is similar, but corresponds to averaging intensity in a pixel window and mapping it orthogonally to a voxel representation on the camera plane in which each pixel intensity value takes up the whole pixel square rather than a single point at the center. Both of these are reasonable models for imaging objects at a medium, uniform distance from a camera \cite{stockman2001computer}.
\end{remark}

\begin{remark}
    For $d=1,2$, the upper bound of \Cref{COR:NpixelCompact} results in better decay compared to results of empirical sampling; i.e., it does not contain a $\log(N)^\frac12$ factor for $d=2$ and achieves a rate of $N^{-1}$ rather than $N^{-\frac12}$ for $d=1$. 
\end{remark}

Our result above is quite general as it holds for arbitrary choice of $\tau_i$. In a concrete application, one has the flexibility to choose them to suit the task at hand. One could choose a given approximation scheme as in \Cref{SEC:Examples} if the measure is absolutely continuous (cf. \Cref{thm:piecewise-approximation}), or one could choose $\tau_i$ to mitigate computational cost without sacrificing accuracy.

\subsection{Non-Compactly Supported Measures}\label{SEC:NonCompact}

So far, we have assumed that $\mu$ is compactly supported in a ball $\BB_R \subset \R^d$ to obtain $N$-term approximation rates, but here we extend the results above to non-compactly supported measure with suitable decay.  If the measure $\mu$ decays fast enough away outside of a ball $\BB_R$, we first estimate $\mu$ with a compactly supported measure $\widehat{\mu}$, and then apply our approximation schemes above to $\widehat\mu$.

We want to create a non-optimal coupling that will send $\mu$ to itself when restricted to sets inside $\BB_R$ but that will project the part of $\mu$ outside of $\BB_R$ to the boundary of the ball.  To do this, we define the projection operator
\begin{align*}
    P_{\BB_R}(x) = \underset{y\in \BB_R}\argmin\;  | x - y |.
\end{align*}
In particular, given any set $B \subset \R^d$, this projection operator has a preimage
$
    P_{\BB_R}^{-1} (B) = \{ x \in \R^d : P_{\BB_R}(x) \in B \}.
$
If $B \cap \BB_{R} = \emptyset$, then, $P_{\BB_R}^{-1}(B) = \emptyset$.  Finally, notice that
\begin{align*}
    P_{\BB_R}^{-1}(B) = P_{\BB_R}^{-1}(B \cap \BB_R) \cup P_{\BB_R}^{-1}(B \cap \partial \BB_R)
\end{align*}
We use this definition in our construction to define the following coupling: 
\begin{align*}
    \pi = \big( I \times P_{\BB_R} \big)_\# \mu = \mu\times (P_{\BB_R})_\# \mu =: \mu\times \widehat\mu.
\end{align*}
Notice first that this coupling sends $\mu$ to itself when restricted to sets in $\BB_R$. Secondly, for $A \subset \BB_R^c$, it projects the measure of $A$ to the boundary $\partial \BB_R$.  In essence, this acts as approximation through a truncated measure supported on the ball $\BB_R$.  In particular, the measure $\pi(\R^d, B) = \widehat{\mu}(B)$ is supported entirely on the ball $\BB_R$.

\begin{proposition}\label{PROP:noncompact_ext}
    Let $\mu\in\W_p(\R^d$), $p\in[1,\infty)$, and $\eps>0$. If $\int_{\BB_R^c}(|x|-R)^p_+\dd\mu(x)<\eps^p$, then $\W_p(\mu,\widehat{\mu})<\eps$, where $\widehat{\mu}$ is the compactly supported measure ${(P_{\BB_R})}_\# \mu$ defined above.
\end{proposition}

\begin{proof}
Notice that
\begin{align*}
    \W_p(\mu, \widehat{\mu})^p &\leq \int_{\R^d \times \R^d} \vert x - P_{\BB_R}(x) \vert^p \dd\mu(x) = \underbrace{\int_{\BB_R} \vert x - P_{\BB_R}(x) \vert^p \dd\mu(x)}_{0} + \int_{\BB_R^c} \vert x - P_{\BB_R}(x) \vert^p \dd\mu(x) \\
    &= \int_{\BB_R^c} \bigg\vert x - R \frac{x}{\vert x \vert} \bigg\vert^p \dd\mu(x) = \int_{\BB_R^c} \bigg( 1 - \frac{R}{\vert x \vert} \bigg)^p \vert x \vert^p \dd\mu(x) = \int_{\BB_R^c} \big( \vert x \vert - R \big)^p \dd\mu(x) < \eps^p,
\end{align*}
and the conclusion follows.
\end{proof}

Now, we can use the theorems above to approximate $\widehat{\mu}$ with the Voronoi cell approximations.

\begin{corollary}\label{COR:Noncompact}
    Take the notations and assumptions of \Cref{THM:Nonuniform} and \Cref{PROP:noncompact_ext}. If $h_X\leq CN^{-\frac1d}$ for some $0<C<\infty$, and $q_X \geq RN^{-\frac1d} + CN^{-\frac2d}$, then $\widehat\mu_X = \sum_{i} \frac{\widehat{\mu}(V_{i})}{\tau_i(V_i)} \tau_i$ is an at most $N$-term approximation to $\mu$ that satisfies
    \begin{align*}
        \W_p(\mu, \widehat\mu_X ) & \leq 2CN^{-\frac{1}{d}} + \eps.
    \end{align*}
\end{corollary}

\begin{proof}
Combine \Cref{THM:Nonuniform}, \Cref{PROP:noncompact_ext}, and \Cref{LEM:PointPacking}. 
\end{proof}

\subsection{Examples}
While the condition $\int_{\BB_R^c}(|x|-R)_+^p\dd\mu(x)<\eps^p$ is quite general, we would like to give some more structured examples to illustrate its conclusion. To begin, we consider what happens when the target measure $\mu$ has extra moments.

\begin{corollary}\label{COR:NoncompactMoment}
    Take the notations and assumptions of \Cref{THM:Nonuniform} and \Cref{PROP:noncompact_ext}. Let 
    $\mu\in\W_{p+q}(\R^d)$ for some $p\in[1,\infty)$ and $q>0$. If 
    $h_X\leq CN^{-\frac{1}{d}}$ for some $0<C<\infty$ and 
    $q_X(N^{\frac{1}{d}}-1)\geq R+h_X$ for some $R>0$, then 
    $\widehat{\mu}_X = \sum_{i} \frac{\widehat{\mu}(V_i)}{\tau_i(V_i)}\tau_i$ is an at-most-$N$-term 
    approximation to $\mu$ that satisfies
    \[
        \W_p(\mu,\widehat{\mu}_X) 
        \;\leq\; 2CN^{-\frac{1}{d}} 
        \;+\; R^{-\frac{q}{p}}\,M_{p+q}(\mu)^{\frac1p}.
    \]
\end{corollary}

\begin{proof}
    For $|x|>R$ we have $(|x|-R)^p\leq |x|^p \leq R^{-q}|x|^{p+q}$, hence
    \[
        \int_{\BB_R^c}(|x|-R)^p_+\,\dd\mu(x)
        \;\leq\; R^{-q}\!\int_{\BB_R^c}|x|^{p+q}\,\dd\mu(x)
        \;\leq\; R^{-q}M_{p+q}(\mu),
    \]
    and \Cref{PROP:noncompact_ext} gives
    \[
        \W_p(\mu,\widehat{\mu}) \leq R^{-\frac{q}{p}}M_{p+q}(\mu)^{\frac1p}.
    \]
    The hypothesis $q_X(N^{\frac{1}{d}}-1)\geq R+h_X$ and \Cref{LEM:PointPacking} give
    \[
        N_{X,R} \leq \left(\frac{R+h_X+q_X}{q_X}\right)^{d} \leq N.
    \]
    and consequently $\widehat{\mu}_X$ has at most $N$ nontrivial terms. 
    \Cref{THM:Nonuniform} applied to $\widehat{\mu}$ on the cells of $X$ 
    yields $\W_p(\widehat{\mu},\widehat{\mu}_X) \leq 2h_X \leq 2CN^{-\frac{1}{d}}$. 
    Combining these estimates yields the claim.
\end{proof}

\begin{remark}\label{REM:NoncompactMomentRate}
    Taking $R$ as large as the cell count constraint allows, namely 
    $R = q_X(N^{\frac{1}{d}}-1) - h_X$, and assuming $N$ is large 
    enough that this value exceeds $\frac12q_X N^{\frac{1}{d}}$ (equivalently, 
    $N^{\frac{1}{d}}\geq 2+2h_X/q_X$), \Cref{COR:NoncompactMoment} 
    specializes to
    \[
        \W_p(\mu,\widehat{\mu}_X) 
        \;\leq\; 2CN^{-\frac{1}{d}} 
        \;+\; 2^{\frac{q}{p}}\,q_X^{-\frac{q}{p}}\,M_{p+q}(\mu)^{\frac1p}\,N^{-\frac{q}{pd}}.
    \]
    The first term is the quantization contribution from the 
    fixed Voronoi decomposition, and the second is the tail decay contribution governed by the extra moment assumption on $\mu$.
\end{remark}

The corollary above assumes a fixed, nonadaptive set of centers $X$, which we have motivated previously as occurring in some applications due to constraints on placements of sensors. The bound there contains a necessary $O(N^{-\frac1d})$ term owing to quantization at the grid. However, if one has the freedom to scale the centers $X$ depending on $N$ (e.g., in numerical quadrature or adaptive sampling schemes) then the quantization and truncation terms can be balanced jointly, and a faster overall rate is obtained. The following corollary assumes that one may place adaptive centers to illustrate this.

\begin{corollary}\label{COR:NoncompactMomentAdaptive}
    Let $\mu\in \W_{p+q}(\R^d)$ for some $p\in[1,\infty)$ and $q>0$, and let 
    $d\in\N$. Then for every $N\in\N$ sufficiently large, there exists a truncation radius $R_N\asymp M_{p+q}(\mu)^{\frac{1}{p+q}}\, N^{\frac{p}{d(p+q)}}$ and a quasi-uniform point set $X_N\subset\R^d$ such that $\widehat{\mu}_{X_N}$ of \Cref{COR:Noncompact} is an at most 
    $N$-term approximation of $\mu$ such that
    \[
        \W_p(\mu,\widehat\mu_{X_N}) 
        \leq C_{p,q,d} M_{p+q}(\mu)^{\frac{1}{p+q}} N^{-\frac{q}{d(p+q)}}.
    \]
\end{corollary}

\begin{proof}
    As in the proof of \Cref{COR:NoncompactMoment}, the moment condition implies \[
        \W_p(\mu,\widehat\mu) \leq M_{p+q}(\mu)^{\frac1p}R^{-\frac{q}{p}}.\]
    For any $R>0$, suppose $X$ is a quasi-uniform set of points with $h_X\leq CN^{-\frac1d}$ and $q_X(N^{\frac{1}{d}}-1)\geq R+h_X$. Taking $C=RN^{\frac1d}/(N^{\frac1d}-2)$, \Cref{COR:Noncompact} yields
    \[
        \W_p(\mu,\widehat\mu_X) \lesssim RN^{-\frac1d} 
        + M_{p+q}(\mu)^{\frac1p}R^{-\frac{q}{p}}.
    \]
    Optimizing the right-hand side over $R$ gives 
    $R_N \asymp M_{p+q}(\mu)^{\frac{1}{p+q}}N^{\frac{p}{d(p+q)}}$, which yields the conclusion. We note that this $R_N$ does satisfy the hypothesis required of \Cref{COR:Noncompact} as well as the bounds above.
\end{proof}

The exponent $\frac{q}{d(p+q)}$ in \Cref{COR:NoncompactMoment}  has a slightly larger exponent compared to that of the empirical measure approximation of Fournier and Guillin \cite{fournier2015rate} because $X_N$ is still a generic quasi-uniform set which may assign too many samples at the tail, which empirical sampling would not. It is possible that this is a legitimate obstruction, and that on certain measures, one would need to concentrate samples in a way that makes the full set of samples $X$ non-quasi-uniform, but  we leave this question to future work.

\section{Conclusion}

This paper has studied structured, deterministic approximation techniques for measures, giving error rates in the Wasserstein-$p$ metric. Here, we summarize our main results and their relation to existing work.

\begin{enumerate}
    \item \textbf{Transfer of approximation rates of densities from $\LL_p$ to $\W_p$ (\Cref{THM:ApproximationTransfer}).} For bounded Lipschitz $\Omega\subset\R^d$ and densities bounded away from zero, we proved that approximation rates in $\LL_p$ transfer directly to approximation rates in $\W_p$, both globally and locally. The proof applies Bogovskii's theorem to the continuity equation in the Benamou--Brenier formulation. As a consequence, the result holds on any bounded Lipschitz domain (in fact, on any bounded John domain), does not require convexity of $\Omega$ nor does it require bounded densities. If additionally, $\Omega$ is convex and the densities are bounded above, \Cref{THM:LowerBound,PROP:BesovSobolevEmbed} exhibits a negative Sobolev space lower bound \[\|f-g\|_{\Sob{-1}{p}_\diamond(\Omega)} \lesssim \W_p(f\dd x, g\dd x),\] extending Peyre's $p=2$ result \cite{peyre2018comparison} to all $p\in(1,\infty)$ via the natural homogeneous negative Sobolev space of mean-zero functions. 

    \item \textbf{Concrete approximation schemes (\Cref{COR:TaylorExample,COR:JohnsonExample,COR:CardinalGaussian,thm:piecewise-approximation,COR:KDE,COR:KDEBeta}).} We exhibit several families of approximation schemes that fit the approximation rate transfer framework of \Cref{THM:ApproximationTransfer}: polynomial approximation, nonstationary shift-invariant space approximation, cardinal Gaussian and inverse multiquadric interpolation, kernel density estimators, and piecewise approximation on quasi-uniform Voronoi meshes. This piecewise framework yields an explicit decomposition of the global Wasserstein error into local contributions, providing a principled foundation for adaptive schemes that allocate finer meshes or higher-order approximators where the target density is harder to approximate.

    \item \textbf{Deterministic $N$-term approximation for arbitrary measures (\Cref{THM:Nonuniform,COR:NonuniformCompact,COR:NpixelCompact}).} We give a flexible framework for deterministic approximations of measures that achieves the rate $\W_p(\mu, \mu_X) \leq 2h_X$ for any measure $\mu\in\W_p(\R^d)$ at any (random or deterministic) quasi-uniform point set $X$ of mesh norm $h_X$. For compactly supported $\mu$, this yields an $N$-term discrete approximation with $\W_p(\mu, \mu_N) = O(N^{-\frac1d})$, matching the asymptotic rate of optimal quantizers \cite{bucklew1982multidimensional,graf2007foundations} and the rate of empirical measure approximation (which holds in expectation) \cite{fournier2015rate,singh2018minimax}, while making different (and sometimes weaker) assumptions and yielding a non-asymptotic, deterministic bound. The result holds for any choice of approximating measures $\tau_i$ on the Voronoi cells $V_i$. We include as examples of piecewise constant approximations, Diracs and indicator functions. In particular, it provides an error bound for the orthographic pixel-camera models that have motivated recent work in measure-valued image analysis \cite{cloninger2025linearized,hamm2023wassmap,khurana2022supervised,liu2025wasserstein,negrini2024applications}.

    \item \textbf{Extension to non-compactly supported measures (\Cref{PROP:noncompact_ext,COR:Noncompact,COR:NoncompactMoment,COR:NoncompactMomentAdaptive}).} The deterministic $N$-term rate extends to measures on all of $\R^d$ under tail-decay assumptions, which we illustrated through examples involving extra moment assumptions.
\end{enumerate}

\section{Acknowledgments}
The authors thank Oleksandr Vlasiuk for introducing them to optimal quantizers for measures and their relation to Wasserstein distance, and Amir Sagiv for the suggestion to consider kernel density estimators.

Both authors were partially supported by the National Science Foundation under Grant No. DMS-1929284 while visiting the Institute for Computational and Experimental Research in Mathematics (ICERM) for the Workshop on Optimal Transport in Data Science. The authors thank ICERM for their hospitality and the organizers for a stimulating workshop.

KH was partially supported by a Research Enhancement Program grant from the College of Science at the University of Texas at Arlington. Research was sponsored by the Army Research Office and was accomplished under Grant
Number W911NF-23-1-0213. The views and conclusions contained in this document are those of the authors and
should not be interpreted as representing the official policies, either expressed or implied, of the Army Research
Office or the U.S. Government. The U.S. Government is authorized to reproduce and distribute reprints for
Government purposes notwithstanding any copyright notation herein.

\bibliographystyle{plain}


\newpage
\appendix

\section{Proof of \cref{PROP:LinearApproximationCounterexample}}\label{APP:ProofCounterexample}
We begin with a 1-dimensional construction. Let $\Omega = [-2,2]$, $\eps\in(0,\frac12)$, and $H = \frac12+\eps$, $L = 1-H = \frac12-\eps.$ Define densities
\[f_\eps(x) := \begin{cases} H, & -2\leq x \leq -1,\\ 0, & -1 < x < 1,\\ L, & 1\leq x \leq 2,\end{cases} \qquad g_\eps(x):= \begin{cases} L, & -2\leq x \leq -1,\\ 0, & -1 < x < 1,\\ H, & 1\leq x \leq 2.\end{cases} \]
Evidently, these define probability densities on $\Omega$, and one may compute their cumulative density functions ($F_{f_\eps}(x) = \int_{-2}^x f_\eps(t)\dd t$ and similarly for $F_{g_\eps}(x)$) as
\[F_{f_\eps}(x) = \begin{cases}H(x+2), & -2\leq x \leq -1,\\ H, & -1\leq x \leq 1,\\ H+L(x-1), & 1\leq x \leq 2,\end{cases} \qquad F_{g_\eps}(x) = \begin{cases}L(x+2), & -2\leq x \leq -1,\\ L, & -1\leq x \leq 1,\\ L+H(x-1), & 1\leq x \leq 2,\end{cases}\]
whose inverses are
\[F_{f_\eps}^{-1}(y) = \begin{cases} \frac{y}{H}-2, & 0\leq y\leq H, \vspace{1em}\\ \frac{y}{L}+1-\frac{H}{L}, & H < y \leq 1,\end{cases} \qquad F_{g_\eps}^{-1}(y) = \begin{cases} \frac{y}{L}-2, & 0\leq y\leq L, \vspace{1em}\\ \frac{y}{H}+1-\frac{L}{H}, & L < y \leq 1.\end{cases}\]
From this, we have that 
\begin{equation}\label{EQN:CounterexampleAbsDiff} |F_{f_\eps}^{-1}(y) - F_{g_\eps}^{-1}(y) | = \begin{cases}
    y\left(\frac1L-\frac1H\right), & 0\leq y\leq L,\\
    3-\frac{L}{H}, & L < y\leq H,\\
    y\left(\frac1H-\frac1L\right) + \frac{H}{L}-\frac{L}{H}, & H < y\leq 1.
\end{cases} \end{equation}
Next, note that
\begin{equation}\label{EQN:CounterexampleNote1}
    \frac1L-\frac1H = \frac{H-L}{LH} = \frac{2\eps}{(\frac12+\eps)(\frac12-\eps)},
\end{equation}
\begin{equation}\label{EQN:CounterexampleNote2}
\frac{H}{L}-\frac{L}{H} = \frac{H^2-L^2}{LH} = \frac{(H+L)(H-L)}{LH} = \frac{H-L}{LH} = \frac{2\eps}{(\frac12+\eps)(\frac12-\eps)},
\end{equation}
since $H+L=1$, and
\begin{equation}\label{EQN:CounterexampleNote3}
3-\frac{L}{H} = 2\frac{1+4\eps}{1+2\eps}.
\end{equation}
Letting $K:= \frac{2\eps}{(\frac12+\eps)(\frac12-\eps)}$ and applying \eqref{EQN:CounterexampleNote1}--\eqref{EQN:CounterexampleNote3} in \eqref{EQN:CounterexampleAbsDiff}, we have

\begin{align*}\W_p(f_\eps\dd x, g_\eps \dd x)^p & = \int_{0}^1|F_{f_\eps}^{-1}(y) - F_{g_\eps}^{-1}(y)|^p\dd y \\ & = \int_0^L (Ky)^p\dd y + 2^p\frac{(1+4\eps)^p}{(1+2\eps)^p}(H-L) + \int_H^1 (K(1-y))^p\dd y \\ & =: I_1+I_2+I_3. \end{align*}

It is easily seen that \[I_1 = I_3 = \frac{1}{p+1}\frac{2^p\eps^p(\frac12-\eps)^{p+1}}{(\frac12+\eps)^p(\frac12-\eps)^p} = \frac{1}{p+1}\frac{2^p\eps^p(\frac12-\eps)}{(\frac12+\eps)^p} = \Theta(\eps^p), \quad \eps\to0^+.\]

The dominant term, however, is
\[I_2 = 2^p\frac{(1+4\eps)^p}{(1+2\eps)^p}(2\eps) = \Theta(\eps),  \quad \eps\to0^+.\]
Hence, $\W_p(f_\eps\dd x,g_\eps\dd x) = \Theta(\eps^\frac1p),$ $\eps\to0^+$.

On the other hand, 
\[\|f_\eps-g_\eps\|_{\LL_p} = \left(\int_{-2}^2|f_\eps(x)-g_\eps(x)|^p\dd x\right)^\frac1p = \left(2(2\eps)^p\right)^\frac1p = \Theta(\eps), \quad \eps\to0^+,\]
whence $\W_p(f_\eps\dd x,g_\eps\dd x) = \Theta(\eps^\frac1p) = \|f_\eps-g_\eps\|_{\LL_p}^\frac1p,$ $\eps\to0^+$, which proves the first claim. However, for any $\alpha>\frac1p$,
\[\frac{\W_p(f_\eps \dd x, g_\eps \dd x)}{\|f_\eps-g_\eps\|_{\LL_p}^\alpha} = \Theta(\eps^{\frac1p-\alpha}) \to\infty, \; \eps\to0^+,\]
which is the second claim.

To obtain a counterexample in arbitrary dimension $d$, we simply tensor the 1-dimensional measures with the Lebesgue measure on the unit disk $Q=[0,1]^{d-1}$ on the other coordinates. Precisely, let 
\[f_\eps^{(d)}(x_1,x') := f_\eps(x_1)\one_Q(x'), \qquad g_\eps(x_1,x') := g_\eps(x_1)\one_Q(x'),\quad (x_1,x')\in \Omega\times Q. \]
It is straightforward to see that $\|f_\eps^{(d)}-g_\eps^{(d)}\|_{\LL_p(\Omega\times Q)} = \|f_\eps-g_\eps\|_{\LL_p(\Omega)} = \Theta(\eps)$, $\eps\to0^+.$ 

Next, we claim that 
\begin{equation}\label{EQN:CounterexampleWpEqual}
    \W_p(f_\eps^{(d)}\dd x,g_\eps^{(d)}\dd x) = \W_p(f_\eps\dd x, g_\eps\dd x),
\end{equation}
where the left-hand side is the Wasserstein distance on $\Omega\times Q\subset\R^d$ and the right-hand side is that on $\Omega\subset\R$. To see the lower bound, note that the coordinate restriction map $\pi_1:\R^d\to\R$ via $x\mapsto x_1$ is a 1-Lipschitz function. Consequently,
\[\W_p(f_\eps\dd x, g_\eps\dd x) = \W_p((\pi_1)_\# f_\eps^{(d)}\dd x, (\pi_1)_\# g_\eps^{(d)}\dd x) \leq \W_p(f_\eps^{(d)}\dd x, g_\eps^{(d)}\dd x).\]
To see the upper bound, suppose that $\pi\in\Pi(f_\eps \dd x, g_\eps \dd x)$ be an optimal transport plan. Then define $\gamma\in\Pi(f_\eps^{(d)}\dd x, g_\eps^{(d)}\dd x)$ via
\[\dd\gamma((x_1,x'),(y_1,y')):= \dd\pi(x_1,y_1)\one_Q(x')\delta_{x'}(y')\dd x'.\]
This is easily seen to be a probability measure on the product space $(\Omega\times Q)\times (\Omega\times Q)$. Note that 
\begin{align*}
    \int_{A\times(\Omega\times Q)}\dd\gamma(x,y) & = \int_\Omega\int_Q \one_A(x_1,x')\dd x'\dd\gamma(x_1,y_1)\\
    & = \int_\Omega\int_Q \one_A(x_1,x')\dd x' f_\eps(x_1)\dd x_1\\
    & = \int_A f_\eps(x_1)\one_Q(x') \dd x'\dd x_1\\
    & = \int_A f_\eps^{(d)}(x)\dd x,
\end{align*}
so the first marginal is as expected. A similar computation shows that the second marginal is $g_\eps^{(d)}\dd x$, whence $\gamma$ is a valid coupling.

Finally, we see that
\begin{align*}
    \W_p(f_\eps^{(d)}\dd x,g_\eps^{(d)}\dd x)^p & \leq \int_{(\Omega\times Q)\times(\Omega\times Q)} |x-y|^p\dd\gamma(x,y)\\
    & = \int_\Omega\int_Q |(x_1,x')-(y_1,x')|^p\dd x'\dd \pi(x_1,y_1)\\
    & = \int_\Omega |x_1-y_1|^p\dd\pi(x_1,y_1)\\
    & = \W_p(f_\eps\dd x,g_\eps\dd x)^p,
\end{align*}
which shows the desired bound. Putting these together, we see that \eqref{EQN:CounterexampleWpEqual} holds.

Our choice of $\Omega$ was not arbitrary. To complete the proof, we simply note that for a scaling factor $c>0$, $\W_p(cf_\eps^{(d)}(c\cdot)\dd x,cg_\eps^{(d)}(c\cdot)\dd x) = c^{-d}\W_p(f_\eps^{(d)}\dd x,g_\eps^{(d)} \dd x)$, and for any translation vector $t$, $\W_p(f_\eps^{(d)}(\cdot-t)\dd x,g_\eps^{(d)}(\cdot-t)\dd x) = \W_p(f_\eps^{(d)}\dd x,g_\eps^{(d)}\dd x)$. Thus for any given $\Omega$, we choose a constant $c$ and translation $t$ such that translating and shrinking the support of the densities constructed here yields densities with support inside $\Omega$. The conclusions above still hold albeit with different asymptotic constants that depend on $c$ and $d$, but note that these are fixed. \hfill\qedsymbol

\section{Proof of \Cref{PROP:BesovSobolevEmbed}}\label{APP:PropBesovProof}

By duality
\[
\|h\|_{B^{-1}_{p,\infty}(\Omega)}
=
\sup\left\{
\left|\int_\Omega h\psi\dd x\right|:
\psi\in B_1^1(\LL_{p'}(\Omega)),\ \|\psi\|_{B_1^1(\LL_{p'}(\Omega))}\leq 1
\right\}.
\]
Next, \cite[Theorem 2.9(ii)]{Triebel2020Theory} implies that $B_1^1(\LL_{p'}(\R^d))\hookrightarrow \Sob{1}{p'}(\R^d)$ (identified with the Triebel--Lizorkin space $F_{p',2}^1(\R^d)$). Let $\psi$ be an admissible function in the set above, and let $\bar\psi:=|\Omega|^{-1}\int_\Omega\psi\dd x$. Since $h$ is mean-zero, we have
\[
\left|\int_\Omega h\psi\dd x\right|
=\left|\int_\Omega h(\psi-\bar\psi)\dd x\right|.
\]

Next, since $\int_\Omega(\psi-\bar\psi)\dd x=0$ and $\nabla(\psi-\bar\psi)=\nabla\psi$,
we see that $\psi-\bar\psi\in\Sob{1}{p'}_\diamond(\Omega)$ with
$\|\nabla(\psi-\bar\psi)\|_{\LL_{p'}(\Omega)}=\|\nabla\psi\|_{\LL_{p'}(\Omega)}$.
It remains to bound $\|\nabla\psi\|_{\LL_{p'}(\Omega)}$ in terms of
$\|\psi\|_{B^1_1(\LL_{p'}(\Omega))}$. Let $E:B_1^1(\LL_{p'}\Omega))\to
B_1^1(\LL_{p'}(\R^d))$ be an extension operator satisfying $E(\psi)|_\Omega=\psi$ and
$\|E(\psi)\|_{B_1^1(\LL_{p'}(\R^d))}\leq C\|\psi\|_{B_1^1(\LL_{p'}(\Omega))}$.
Since $E(\psi)|_\Omega=\psi$, we have
$\nabla\psi=\nabla(E(\psi))|_\Omega$ a.e., whence
\[
\|\nabla\psi\|_{\LL_{p'}(\Omega)}
\leq\|\nabla(E(\psi))\|_{\LL_{p'}(\R^d)}
\leq\|E(\psi)\|_{\Sob{1}{p'}(\R^d)}\leq C\|E(\psi)\|_{B_1^1(\LL_{p'}(\R^d))}\leq C\|\psi\|_{B_1^1(\LL_{p'}(\Omega)},
\]
where the first inequality simply comes from extending the integral past $\Omega$, the second is by definition of the Sobolev norm, the third by the embedding theorem, and last by property of the extension operator.

It follows that \eqref{EQ:NegativeSobolevNorm} and that
\begin{align*}
    \left|\int_\Omega h\psi \dd x\right|  = \left|\int_\Omega h(\psi-\bar\psi)\dd x\right| \leq \|h\|_{\Sob{-1}{p}_\diamond(\Omega)}
\|\nabla(\psi-\bar\psi)\|_{\LL_{p'}(\Omega)}
& \leq C\|h\|_{\Sob{-1}{p}_\diamond(\Omega)}\|\psi\|_{B^1_{1}(\LL_{p'}(\Omega))}\\
& \leq C\|h\|_{\Sob{-1}{p}_\diamond(\Omega)}.
\end{align*}
Taking the supremum over all admissible $\psi$ gives the result. \hfill \qedsymbol

\end{document}